\documentclass{article}

 \usepackage[preprint]{neurips_2026}

\usepackage[colorlinks=true, citecolor=gray, linkcolor=red]{hyperref}
\usepackage[utf8]{inputenc} 
\usepackage[T1]{fontenc}    
\usepackage{hyperref}       
\usepackage{url}            
\usepackage{booktabs}       
\usepackage{amsfonts}       
\usepackage{caption}
\usepackage{nicefrac}       
\usepackage{microtype}      
\usepackage{xcolor}         
\usepackage{graphicx}
\usepackage{multirow}
\usepackage{subcaption}
\usepackage[table]{xcolor}
\usepackage{amsmath, amsthm}
\usepackage{amssymb} 
\usepackage[most]{tcolorbox}
\usepackage{enumitem}
\usepackage{wrapfig}
\usepackage{algorithm}
\usepackage{algpseudocode}
\usepackage{mathtools}

\definecolor{colExHigh}{RGB}{200,  0,   0}   
\definecolor{colHigh}  {RGB}{230, 100,   0}  
\definecolor{colMed}   {RGB}{180, 130,   0}  
\definecolor{colLow}   {RGB}{ 60, 100, 200}  

\definecolor{stagecolor}{RGB}{235,235,235}
\definecolor{answerbg}  {RGB}{173,216,230}
\definecolor{wrongbg}   {RGB}{255,220,220}
\definecolor{correctbg} {RGB}{210,240,210}

\newcommand{\tokEH}[1]{\textcolor{colExHigh}{\textbf{#1}}}  
\newcommand{\tokH}[1] {\textcolor{colHigh}  {\textbf{#1}}}  
\newcommand{\tokM}[1] {\textcolor{colMed}   {\textbf{#1}}}  
\newcommand{\tokL}[1] {\textcolor{colLow}          {#1}}    

\newcommand{\wrongans}[1] {\colorbox{wrongbg}  {\textbf{#1}}}
\newcommand{\correctans}[1]{\colorbox{correctbg}{\textbf{#1}}}
\newtheorem{definition}{Definition}[section]
\newtheorem{proposition}{Proposition}[section]
\newtheorem{theorem}{Theorem}[section]
\newtheorem{corollary}{Corollary}[section]

\theoremstyle{remark}
\newtheorem{remark}{Remark}[section]
\newtcolorbox{keyfindings}{
    enhanced,                  
    title=Key Findings,        
    fonttitle=\bfseries,       
    coltitle=white,            
    colback=blue!5,            
    colframe=black,            
    boxrule=0.8pt,             
    arc=4pt,                   
    attach boxed title to top left={xshift=0.5cm, yshift=-\tcboxedtitleheight/2, yshifttext=-\tcboxedtitleheight/2},
    boxed title style={
        colback=black,         
        colframe=black,        
        arc=3pt,               
        boxrule=0pt,
    },
    left=5pt, right=5pt, top=5pt, bottom=5pt
}
\title{Focus on the Core: Empowering Diffusion Large Language Models by Self-Contrast}

    \author{Jinyuan Feng$^{1,2}$ \quad Xin Yu$^{1}$ \quad Yiqun Chen$^{3}$ \quad  \textbf{Xiaochi Wei}$^{4}$ \\\quad \textbf{Yan Gao}$^{4}$ \quad \textbf{Yi Wu}$^{4}$ \quad \textbf{Yao Hu}$^{4}$ \quad \textbf{Zhiqiang Pu}$^{1,2}$ \\
$^{1}$Institute of Automation, Chinese Academy of Sciences \\
$^{2}$School of Artificial Intelligence, University of Chinese Academy of Sciences\\  
$^{3}$Renmin University of China \\ 
$^{4}$Xiaohongshu Inc. \\ 
}

\begin{document}
\maketitle

\begin{abstract}
  The iterative denoising paradigm of Diffusion Large Language Models (DLMs) endows them with a distinct advantage in global context modeling. However, current decoding strategies fail to leverage this capability, typically exhibiting a local preference that overlooks the heterogeneous information density within the context, ultimately degrading generation quality. To address this limitation, we systematically investigate high-information-density (HD) tokens and present two key findings: (1) explicitly conditioning on HD tokens substantially improves output quality; and (2) HD tokens exhibit an early-decoding tendency, converging earlier than surrounding tokens. Motivated by these findings, we propose Focus on the Core \textbf{(FoCore)}, a training-free decoding strategy that utilizes HD tokens in a self-contrast manner, wherein HD tokens are temporarily remasked as negative samples, to guide generation. We further introduce FoCore\_Accelerate \textbf{(FoCore\_A)}, an efficient variant that, upon detecting HD token convergence, performs parallel decoding over stable candidates within a local context window, substantially accelerating generation. Extensive experiments on math, code and logical reasoning benchmarks demonstrate that FoCore consistently improves generation quality and efficiency across both LLaDA and Dream backbones. For instance, on HumanEval, FoCore improves pass@1 from 39.02 to 42.68 over standard Classifier-Free Guidance, while FoCore-A reduces the number of decoding steps by 2.07× and per-sample latency from 20.76s to 8.64s (-58.4\%).
\end{abstract}
\section{Introduction}
\label{sec:introduction}
By replacing sequential generation with iterative denoising, diffusion large language models (DLMs) present an alternative to autoregressive (AR) architectures, enabling global context modeling and flexible decoding~\citep{nie2025large,ye2025dream}. However, this advantage is frequently bottlenecked by locally-biased decoding strategies: candidate tokens in closer proximity to already-denoised positions tend to receive disproportionately higher confidence scores~\citep{fang2026locally,cai2026confidence}. This locally-biased paradigm fundamentally neglects the heterogeneous information density inherent in natural language, as it confines the model's attention to a narrow positional neighborhood, failing to capture core reasoning tokens within the global context~\citep{ma2026mask}. \emph{Therefore, rather than passively sampling high-confidence tokens locally, proactively prioritizing core reasoning units is essential to fully unleash the contextual capabilities of DLMs.}

In this paper, we explore how core reasoning tokens influence generation quality through a novel perspective of "Heterogeneous Information Density" within DLMs. Our analysis reveals a significant heterogeneity in the information density of contextual tokens, where the majority of ordinary tokens function merely as redundant background padding, designated as \emph{"low-information-density (LD)" tokens}, whereas a minority of core tokens act as logical anchors with exceptionally \emph{"high-information-density", termed HD tokens}. During the denoising processes of DLMs, HD tokens serve as pivotal anchors that steer the reasoning toward correct trajectories. As illustrated by the qualitative analysis in Figure~\ref{fig:introduction}(a), HD tokens effectively constrain the denoising process when the model encounters divergent reasoning paths. For instance, in the depicted age calculation problem, the temporal preposition "in" \emph{("Sarah will be 20 years old \textcolor{red}{in} 4 years")} functions as a quintessential HD token. This HD token dictates the core logical branch for deriving the current age, such that correctly capturing it steers the model toward the correct path \emph{(identifying the current age as 16 to yield the final answer of 22)}, whereas neglecting it risks misguiding the reasoning into an erroneous addition \emph{(e.g., 20+4=24)}.

To quantitatively validate the positive impact of HD tokens, we designed comparative experiments as depicted in Figure~\ref{fig:introduction}(b). Specifically, we evaluated the performance disparity between a decoding strategy prioritizing HD tokens (Mask HD) and an alter
prioritizing LD tokens (Mask LD). Here, the prioritization of tokens is achieved via Classifier-Free Guidance~\citep{ho2022classifier}, which amplifies the model's sensitivity by contrasting conditional and unconditional denoising trajectories. The results indicate that the Mask HD consistently outperformed Mask LD in terms of accuracy (Acc@1) across all mathematical sub-domains. Notably, this strategy achieved a maximum accuracy increase of +2.07\% (from 0.579 to 0.599) in the Prealgebra task and a +1.48\% improvement in Counting \& Probability, while maintaining stable positive gains even in the more challenging Algebra (+0.76\%) and Precalculus (+0.55\%) tasks. In essence, it is HD tokens, rather than LD counterparts, that constitute the informational backbone governing complex reasoning trajectories, underscoring that selectively prioritizing high-density information is effective for optimizing generation quality.
\begin{figure}[t!]
    \centering
    \includegraphics[width=\textwidth]{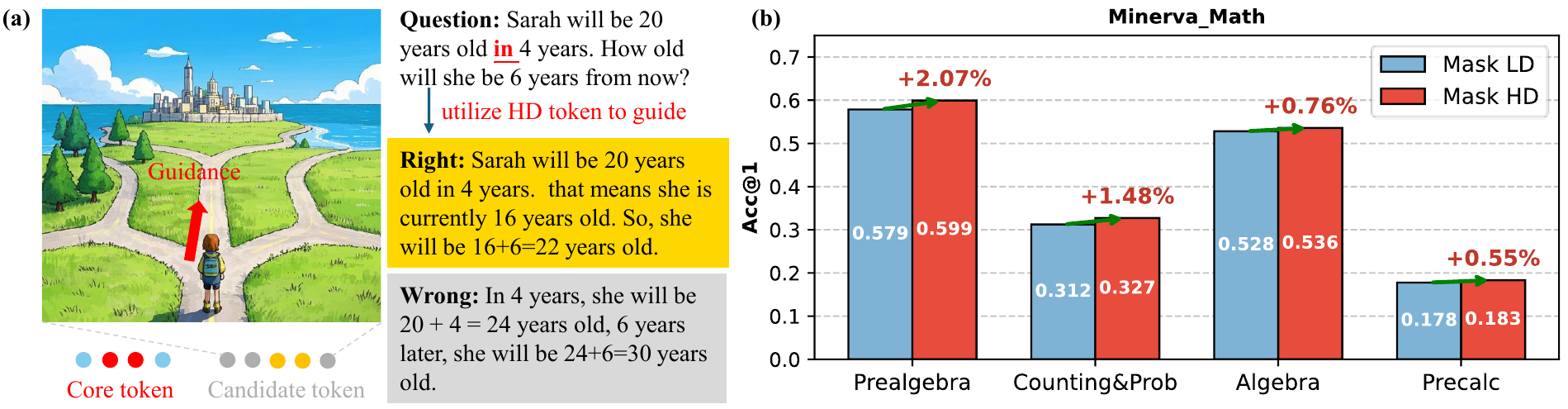}
    \caption{(a) A qualitative example illustrating that focusing on HD tokens effectively guides correct reasoning generation. 
    (b) Quantitative results on Minerva\_Math show that conditioning on HD tokens (\textbf{Mask HD}) consistently outperforms LD tokens (\textbf{Mask LD}).}
    \label{fig:introduction}
\end{figure}

Motivated by the aforementioned validations, we propose \textbf{FoCore} 
(\textbf{Fo}cus on the \textbf{Core}), a training-free decoding strategy 
that leverages HD tokens in a self-contrast manner. Specifically, FoCore 
dynamically identifies and temporarily remasks converged HD tokens throughout 
the denoising trajectory to construct contrastive negative samples, and steers the generation by contrasting the output logits between the original and masked conditions, thereby compelling the model to strictly anchor its attention on core HD tokens during subsequent denoising steps. We further dive deep into the decoding dynamics of HD tokens throughout the denoising process, revealing an intriguing property: \emph{HD tokens exhibit an early-decoding tendency, converging to their final states earlier than surrounding tokens}. Motivated by this finding, we introduce \textbf{FoCore\_A}, an efficient variant that, upon detecting HD token convergence, performs parallel decoding over stable candidates within a local context window, substantially accelerating generation. Our contributions are threefold:
\begin{itemize}
    \item \textbf{Empirical analysis on HD tokens:} We systematically investigate the behavioral properties of HD tokens and derive three critical insights that collectively demonstrate their dual benefits in generation quality and decoding efficiency, serving as the motivation and supporting evidence for the design of FoCore and FoCore\_A.
    \item \textbf{A decoding strategy leveraging HD tokens:} We propose 
FoCore, a training-free decoding strategy that dynamically remasks converged HD tokens as negative samples to enhance generation quality, and FoCore\_A, an efficient variant that exploits their early-decoding tendency for parallel decoding to further accelerate generation.
    \item \textbf{Comprehensive experimental evaluation:} Extensive experiments on various benchmarks demonstrate that FoCore consistently enhances generation quality across both LLaDA and Dream backbones, while FoCore\_A reduces per-sample latency from 20.76s to 8.64s ($-$58.4\%) 
with a 2.07$\times$ decoding step reduction over standard Classifier-Free Guidance.
\end{itemize}
\section{An Exploratory Analysis on DLMs}
\label{sec:explore}
\subsection{Preliminary on DLMs}
DLMs formulate text generation as an iterative denoising process, centered on two probabilistic processes applied to a clean sequence $x_0 \sim p_{\text{data}}(x_0)$ of length $L$: a forward noising process that progressively corrupts $x_0$ 
into noise, and a reverse denoising process that learns to reconstruct $x_0$ from the corrupted input. Due to space constraints, additional related work is provided in Appendix~\ref{sec:related_work}.

\textbf{Forward Noising Process.} In the discrete space, the forward process is a Markov chain progressively introducing an absorbing state, \texttt{[MASK]} (denoted as $m$)~\citep{austin2021structured}. By unrolling the step-wise transitions $q(x_t|x_{t-1})$, the marginal probability at any step $t \in [0, T]$ is analytically condensed into a categorical distribution governed by a monotonically decreasing schedule $\alpha_t$:
\begin{equation}
    q(x_{1:T} | x_0) = \prod_{t=1}^T q(x_t | x_{t-1}) \quad \Longrightarrow \quad q(x_t | x_0) = \text{Cat}(x_t; \alpha_t x_0 + (1 - \alpha_t)m).
\end{equation}
This indicates that at step $t$, each token retains its original identity $x_0$ with probability $\alpha_t$, and is corrupted into $m$ with probability $1-\alpha_t$. As $t \to T$, $x_T$ degenerates into a fully masked state.

\textbf{Reverse Generative Process.} The reverse process employs a neural network parameterized by $\theta$ to invert this corruption. Starting from a fully masked input $x_T = (m, \dots, m)$, the iterative generation is rigorously formulated as:
\begin{equation}
    p_\theta(x_{0:T}) = p(x_T) \prod_{t=1}^T p_\theta(x_{t-1} | x_t) = p(x_T) \prod_{t=1}^T q(x_{t-1} | x_t,\hat{x}_0) p_\theta(\hat{x}_0 | x_t).
\end{equation}
This formulation decomposes decoding into two steps: \textbf{(1) Prediction} $p_\theta(\hat{x}_0 | x_t)$, where the model predicts a clean candidate $\hat{x}_0$ based on the noisy $x_t$; and \textbf{(2) Re-masking} $q(x_{t-1} | x_t,\hat{x}_0)$, where the forward transition is re-applied to inject noise into $\hat{x}_0$. In practice, guided by heuristic strategies (e.g., predictive confidence)~\citep{chang2022maskgit}, the re-masking step selectively unmasks and fixes high-certainty tokens while replacing the rest with $m$. By iteratively alternating these two steps (often accelerated via step-skipping), DLMs eschew traditional autoregression, progressively reducing the mask ratio in a globally parallel manner to yield the coherent text $x_0$.

\begin{figure}[htbp]
    \centering 
    
    \begin{subfigure}[b]{0.78\textwidth} 
        \raggedright 
        \includegraphics[width=\textwidth]{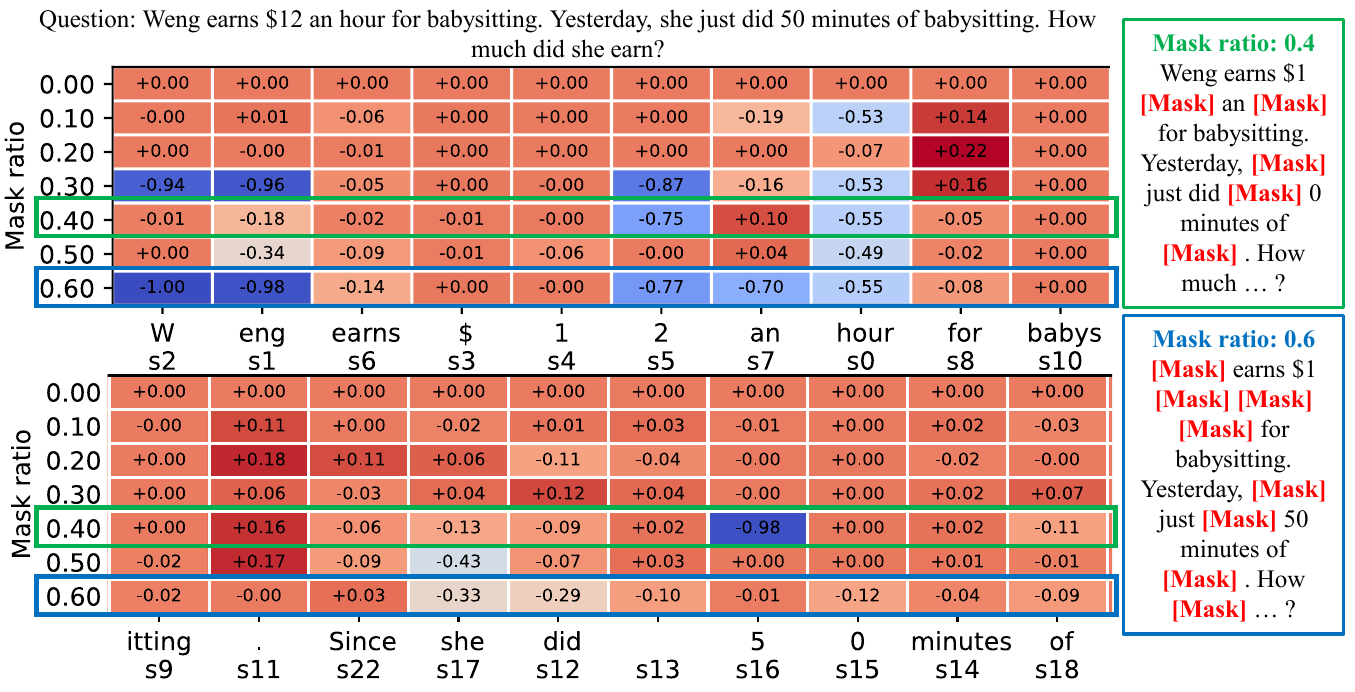}
        \caption{}
        \label{fig:left_pic}
    \end{subfigure}
    \hfill 
    \begin{subfigure}[b]{0.21\textwidth} 
        \raggedleft 
        
        \begin{subfigure}{\textwidth}
            \raggedleft
            \includegraphics[width=\textwidth]{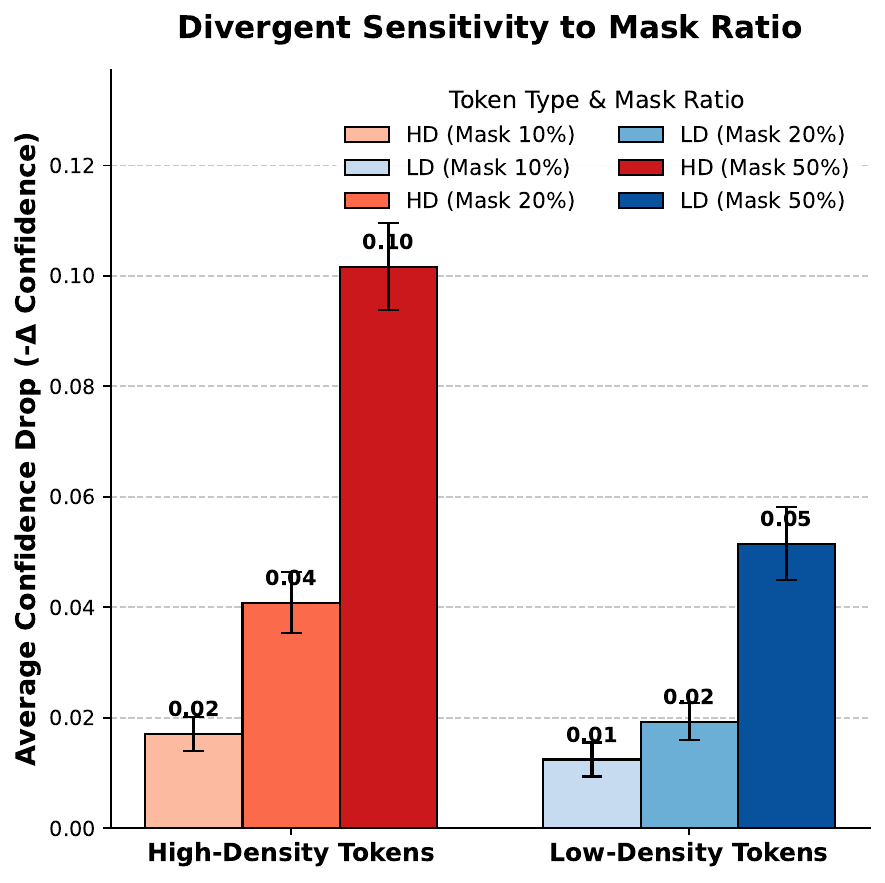}
        \end{subfigure}
        
        \begin{subfigure}{\textwidth}
            \raggedleft
            \includegraphics[width=\textwidth]{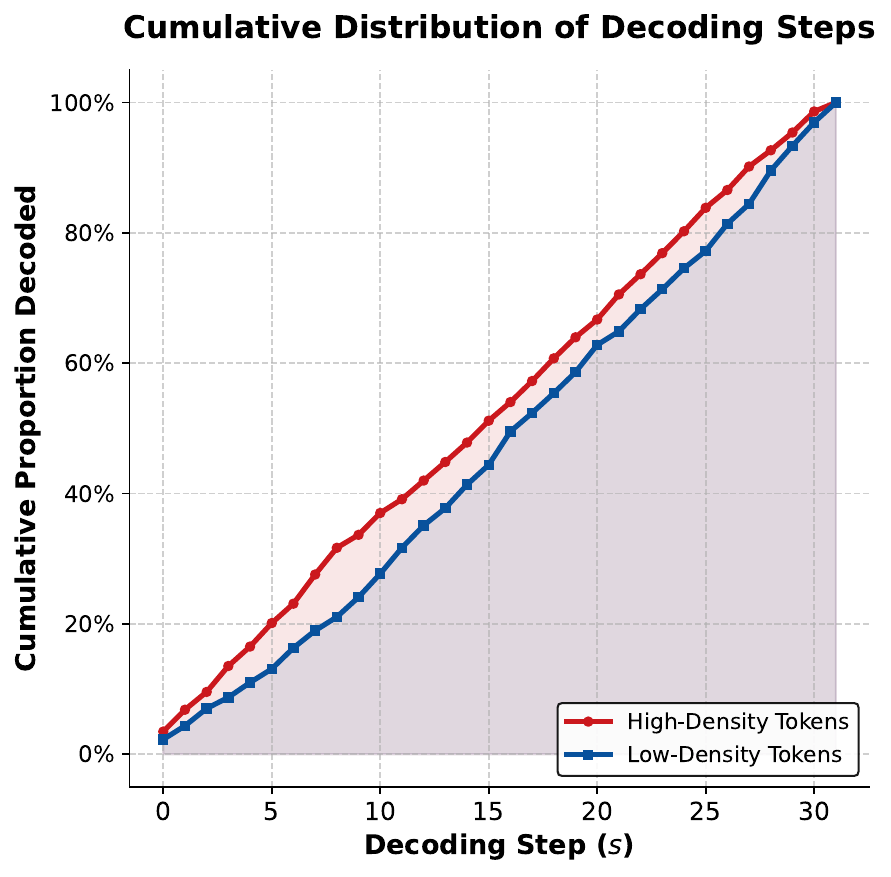}
            \caption{}
            \label{fig:right_bottom}
        \end{subfigure}
    \end{subfigure}

    \caption{Overview of the decoding characteristics of HD tokens: high contextual sensitivity and premature commitment. \textbf{(a) Case Study:} Heatmaps show the confidence drop ($\Delta$ Confidence) under varying global mask ratios. Highlighted HD tokens (e.g., numbers, entities) suffer severe confidence collapse without their corresponding HD tokens in the global context. 
        \textbf{(b) Quantitative Analysis:} The top chart proves that HD tokens are significantly more sensitive to context masking than LD tokens. Counterintuitively, the bottom curve reveals that current strategies prematurely decode these highly sensitive HD tokens at earlier steps. Detailed analysis is provided in Section~\ref{sec:analysis}.}
    \label{fig:overview_explore}
\end{figure}

\subsection{Decoding Dynamics of HD Tokens}
\label{sec:analysis}
To systematically characterize the decoding dynamics of HD tokens,  we conduct qualitative and quantitative analyses (Figure~\ref{fig:overview_explore}) across three critical dimensions:  information density, global context dependency, and decoding timing.These dimensions collectively reveal how HD tokens behave distinctly from other tokens throughout the iterative generation process, yielding three key findings:

\begin{keyfindings}
\begin{itemize}[label=$\circ$, leftmargin=1.5em, itemsep=0.1em]
    \item \textbf{Heterogeneous information density:} When the global context is partially masked, confidence drops ($\Delta$ Confidence) exhibit marked heterogeneity across token types. HD tokens (e.g., entities, numbers) suffer extreme collapses, establishing them as irreplaceable semantic hubs, whereas structural LD tokens remain robust.
    
    \item \textbf{Global context dependency:} The confidence of these core tokens degrades rapidly and non-linearly even under minimal global context masking, proving their strict reliance on a holistic sequence view to be accurately determined.
    
    \item \textbf{Propensity for early decoding:} Confidence-based decoding strategies inherently identify and denoise HD tokens at earlier generation steps. Rather than a structural flaw, this early-decoded HD tokens provides critical \textit{semantic anchors} that can be leveraged to significantly accelerate the subsequent decoding of surrounding LD tokens.
\end{itemize}
\end{keyfindings}

To substantiate these findings with rigorous empirical grounding, we conduct a systematic and in-depth analysis across each identified dimension: 

\textbf{Analysis of Heterogeneous Information Density (Refining Finding~1):} 
To construct the heatmap in Figure~\ref{fig:left_pic}, we systematically obscure the input prompt at varying mask ratios (y-axis) and record the corresponding confidence drops ($\Delta$ Confidence) for each generated token. As visualized, this perturbation reveals a striking dichotomy in information density: HD tokens (e.g., ``W'', ``2'', ``hour'') suffer extreme, deep-blue confidence collapses (approaching $-1.00$) when deprived of global context, whereas structural LD tokens (e.g., ``for'', ``she'') remain robustly unaffected (orange cells). \emph{Crucially, examining specific perturbation scenarios (mask ratios of $0.4$ and $0.6$, green and blue boxes) exposes a profound intra-sequence dependency.} The localized variations in the heatmap strictly align with which specific prompt tokens are masked. For instance, under the $0.4$ mask ratio, obscuring the HD token ``5'' (from ``50 minutes'') in the prompt causes the model's predictive confidence for the corresponding target HD token ``5'' ($s_{16}$) to catastrophically plummet to $-0.98$. This confirms that high-density semantic hubs are intricately coupled; perturbing even a single HD token in the context catastrophically impairs the prediction of related HD tokens.

\textbf{Analysis of Global context dependency: (Refining Finding~2):} 
Figure~\ref{fig:right_bottom} (\textbf{Top}) quantifies sequence-level heterogeneous information density (500 samples) by contrasting the confidence drop ($-\Delta$ Confidence) of HD tokens (e.g., entities, nums) versus LD tokens (e.g., prepositions) across varying global mask ratios. The results demonstrate a distinctly divergent sensitivity: HD tokens suffer rapid, non-linear degradation ($0.02 \rightarrow 0.04 \rightarrow 0.10$), culminating at 50\% masking with a confidence drop ($0.10$) precisely double that of robust LD tokens ($0.05$). This statistically substantiates the \emph{strict global context dependency} inherent in DLMs. It proves that accurately predicting core semantic hubs dictates an entirely uncorrupted holistic sequence view; thus, even minimal contextual perturbations disproportionately penalize HD tokens, inevitably triggering a cascade of subsequent semantic errors.

\textbf{Analysis of Propensity for early decoding (Refining Finding~3):} 
Figure~\ref{fig:right_bottom} (\textbf{Bottom}) tracks the cumulative decoding distribution ($s$), exposing a systematic temporal divergence: driven by local high-confidence heuristics, the model intrinsically prioritizes core semantic hubs, causing HD tokens solidification to consistently outpace LD tokens (e.g., at $s=10$, $\sim$38\% HD vs. $\sim$28\% LD decoded). Crucially, this innate ``core-first'' stabilization unlocks a paradigm for inference acceleration. Once the critical reasoning backbone (HD) is established early, applying computationally expensive iterative decoding to the structurally predictable LD residuals becomes redundant. Consequently, late-stage generation can be aggressively consolidated into a single parallel pass, drastically slashing latency without compromising semantic integrity.
\section{Methodology}
\label{sec:method}
In this section, we present our methodology by answering two questions: (1) how can HD tokens be accurately identified, and (2) how can they be leveraged to improve performance and efficiency? 
\subsection{Identification of HD token}
\begin{wrapfigure}{r}{0.4\textwidth}   
  \centering
  \vspace{-1em}   

  \includegraphics[width=\linewidth]{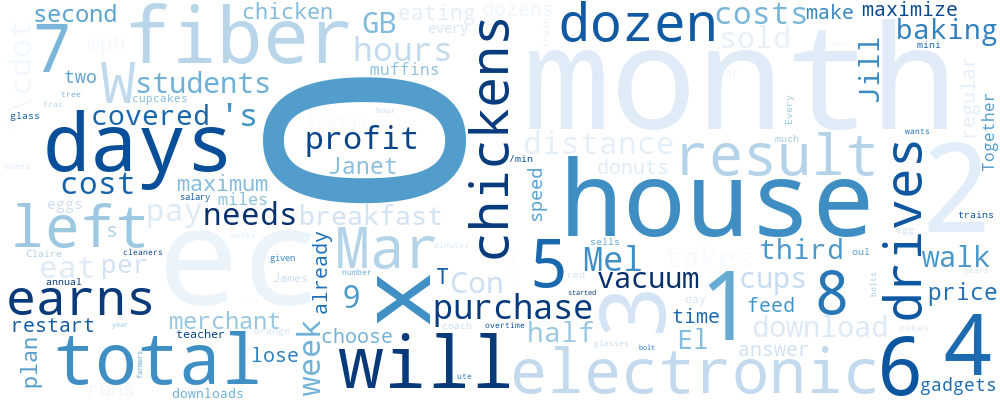}
  \captionof{figure}{HD tokens in MATH}
  \label{fig:math_cloud}

  \vspace{0.4em}

  \includegraphics[width=\linewidth]{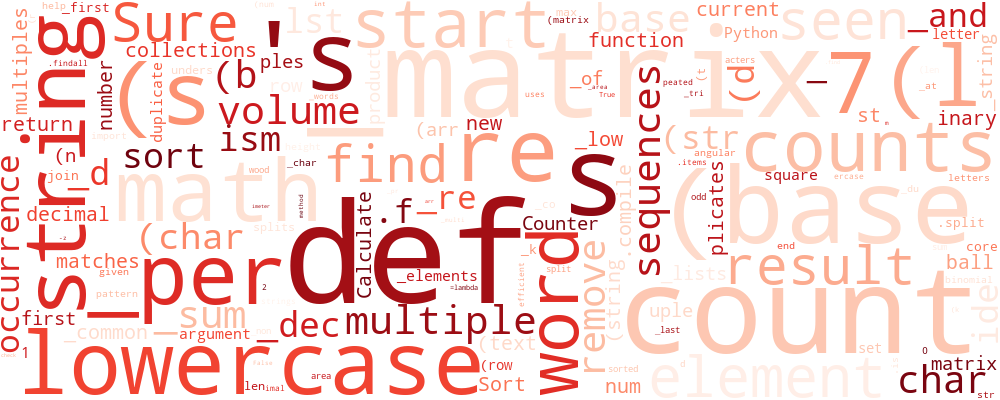}
  \captionof{figure}{HD tokens in MBPP}
  \label{fig:code_cloud}

  \vspace{0.4em}

  \resizebox{\linewidth}{!}{
    \begin{tabular}{lccc}
      \toprule
                  & HumanEval       & MBPP            & BBH             \\
      \midrule
      Mask HD     & 38.41           & 35.00           & 56.12           \\
      Mask LD     & 36.50           & 36.60           & 57.02           \\
      JS\_detect  & \textbf{42.68}  & \textbf{40.40}  & \textbf{58.37}  \\
      \bottomrule
    \end{tabular}
  }
  \captionof{table}{Performance Comparison}
  \label{tab:performance}

  \vspace{-1em}   
\end{wrapfigure}
In our quantitative analysis conducted on Figure~\ref{fig:introduction}(b), we identify HD tokens by leveraging the SpaCy natural language processing library~\citep{honnibal2020spacy} in conjunction with part-of-speech tagging and named entity recognition. This identification mechanism is operationalized through a two-step procedure. Initially, during the text decoding and parsing phase, the token IDs generated by the large language model are decoded into plain text, stripped of leading and trailing whitespace, and processed through the SpaCy \textit{en\_core\_web\_sm} model for syntactic analysis. Subsequently, a dual-rule matching criterion is applied, which dictates that any text segment containing a noun, a proper noun, a number, or constituting a recognized named entity is classified as a HD token. 

While this linguistic-rule-based heuristic approach demonstrates significant efficacy in extracting HD tokens within the MATH dataset, our subsequent evaluation across the code generation and reasoning benchmarks (HumanEval, MBPP, and BBH), reveal its inherent limitations. Specifically, as illustrated in Figure~\ref{fig:math_cloud} and ~\ref{fig:code_cloud}, while HD tokens in the MATH dataset align naturally with linguistic rules 
(e.g., nouns like \texttt{days} and numbers like \texttt{7}, 
\texttt{5}), those in code generation benchmarks such as MBPP exhibit 
substantially different characteristics, encompassing programming keywords 
(e.g., \texttt{def}, \texttt{return}, \texttt{sort}) and code-specific identifiers 
(e.g., \texttt{string}, \texttt{lowercase}, \texttt{Counter}). Such distinction renders the linguistic-rule-based heuristic fundamentally inadequate, exposing its over-reliance on human priors and excessive domain coupling. To this end, achieving domain-agnostic HD token identification necessitates a paradigm shift from externally defined rule sets toward the model's intrinsic perceptual capabilities, motivating us to propose a fully automated approach.

\textbf{Automated Identification via Distribution Stability}\quad The underlying assumption is that if a decoded token is an HD token, the model's predicted probability distribution for that position should exhibit significant divergence across consecutive iterations. As empirically evidenced in Figure~\ref{fig:right_bottom} (\textbf{Top}) and consistent with ~\citep{zou2026breaking}, HD tokens tend to induce notable distributional shifts across decoding steps, whereas LD tokens maintain a relatively stable predictive distribution.

Let $P_t(x_i)$ denote the predicted probability distribution over the 
vocabulary for the $i$-th position at iteration $t$. To quantify the 
discrepancy between the predictions at step $t$ and $t-1$, we employ 
an approximated Jensen-Shannon (JS) divergence. Considering the 
computational overhead of calculating the divergence over the entire 
vocabulary and the noise introduced by long-tail, low-frequency tokens, 
we propose a Top-$K$ truncated JS divergence, which restricts the 
computation to the $K$ most probable tokens at the current step. Formally, let $\mathcal{V}_K = \text{Top-}K\big(P_t(x_i)\big)$ denote 
the set of the $K$ highest-probability tokens under $P_t(x_i)$. 
The step-wise instability $D_t^{(i)}$ for the $i$-th position at 
step $t$ is then defined as:
\begin{equation}
    D_t^{(i)} = \frac{1}{2} \sum_{v \in \mathcal{V}_K} P_t(v) \log 
    \left( \frac{P_t(v)}{M(v)} \right) + \frac{1}{2} \sum_{v \in 
    \mathcal{V}_K} P_{t-1}(v) \log \left( \frac{P_{t-1}(v)}{M(v)} 
    \right)
    \label{eq:js_divergence}
\end{equation}
where $M(v) = \frac{1}{2} \big( P_t(v) + P_{t-1}(v) \big)$. 
A larger $D_t^{(i)}$ indicates a greater distributional shift at that 
position, suggesting that the corresponding token is more likely to 
be a HD token.

\textbf{Temporal Smoothing via Historical Momentum}\quad During the iterative decoding process, the continuous updating of other masked tokens in the sequence inevitably introduces dynamic contextual shifts, which can cause step-wise fluctuations in distributions. Therefore, relying solely on the single-step divergence $D_t^{(i)}$ is prone to misjudgments. To obtain a robust evaluation metric, we introduce an Exponential Moving Average (EMA) to track the historical evolution of the distribution divergence. For each position $i$, we maintain a cumulative instability score $S_t^{(i)}$, updated as follows:
\begin{equation}
    S_t^{(i)} = \alpha \cdot S_{t-1}^{(i)} + (1 - \alpha) \cdot D_t^{(i)}
    \label{eq:ema_smoothing}
\end{equation}
where $\alpha \in [0, 1)$ is the EMA decay coefficient. Through this mechanism, $S_t^{(i)}$ encodes the memory of the token's stability over past iterations. Ultimately, a token with a high cumulative instability 
$S_t^{(i)}$ indicates that the model's predicted distribution for 
this position undergoes significant shifts throughout the iterative 
decoding process, suggesting a strong likelihood of being an HD token. 
As demonstrated in Table~\ref{tab:performance}, our automated 
identification approach (\textbf{JS\_detect}) achieves substantial 
improvements over the linguistic-rule-based heuristic, yielding gains of $4.27$, $5.40$, and $2.25$ points 
on HumanEval, MBPP, and BBH, respectively.

\subsection{Exploitation of HD Tokens}
\begin{figure}[t!]
    \centering
    \includegraphics[width=\textwidth]{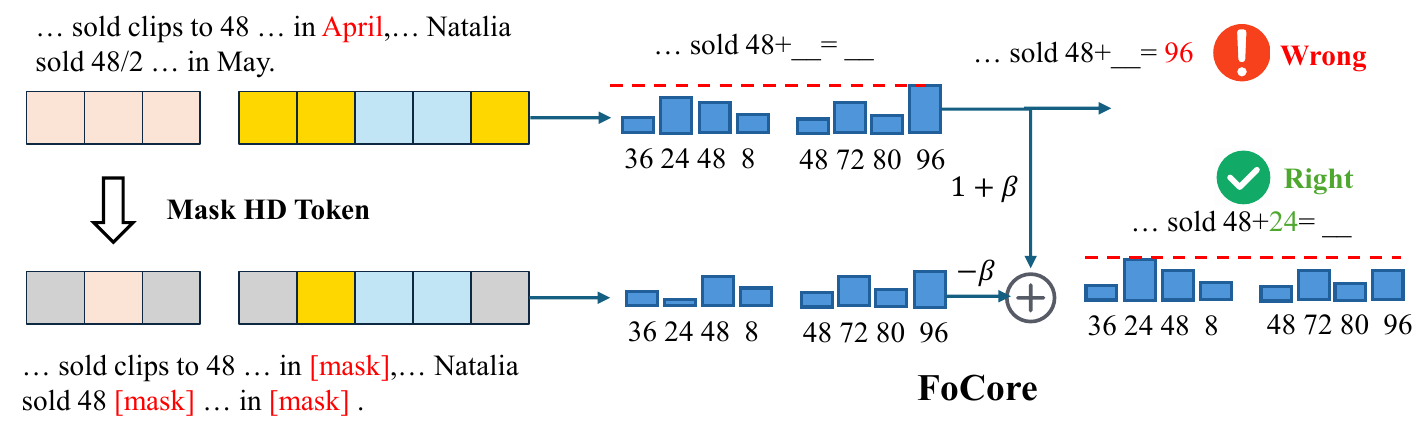}
    \caption{An illustration of the FoCore self-contrast mechanism. FoCore identifies and temporarily remasks HD tokens throughout the denoising trajectory to construct contrastive negative samples. }
    \label{fig:focore}
\end{figure}
As illustrated in Figure~\ref{fig:focore}, this subsection elaborates on how FoCore exploits HD tokens via Self-Contrast: FoCore dynamically identifies and temporarily remasks HD tokens throughout the denoising trajectory to construct contrastive negative samples, and subsequently steers the generation by contrasting the output logits between the original and masked conditions, thereby compelling the model to strictly anchor its attention on core HD tokens.

\textbf{Construction of the negative samples}\quad At iteration $t$, let $X_t$ denote the current sequence state, and $\mathcal{U}_t$ denote the set of indices for currently unmasked tokens. We aim to synthesize a ``unconditional'' input sequence, denoted as $\tilde{X}_t$, by selectively masking the HD tokens of the sequence.

First, using the cumulative instability scores $S_t^{(i)}$ computed previously, we identify the top-$K$ unstable tokens within the decoded sequence. Specifically, we select a subset of indices $\mathcal{K}_t \subset \mathcal{U}_t$ such that:
\begin{equation}
    \mathcal{K}_t = \mathop{\arg\max}_{\mathcal{K} \subset \mathcal{U}_t, |\mathcal{K}|=K} \sum_{j \in \mathcal{K}} \left( S_t^{(j)} + \epsilon \right)
\end{equation}
where $\epsilon \sim \mathcal{U}(0, 10^{-8})$ is an infinitesimal uniform noise added to break ties. Subsequently, we construct the unconditional input $\tilde{X}_t$ by modifying the current sequence $X_t$. For the non-\texttt{[MASK]} tokens whose indices belong to $\mathcal{K}_t$, we explicitly replace them with the special \texttt{[MASK]} token, while the rest of the sequence (including the stable tokens and the masked tokens) remains unchanged:
\begin{equation}
    (\tilde{X}_t)_j = 
    \begin{cases} 
        \texttt{[MASK]}, & \text{if } j \in \mathcal{K}_t \\ 
        (X_t)_j, & \text{otherwise.} 
    \end{cases}
    \label{eq:piecewise_masking}
\end{equation}
Note that when all instability scores are equal (e.g., $S_t^{(i)} = 0$ at early iterations), the selection of $\mathcal{K}_t$ is governed solely by $\epsilon$, reducing the mechanism to random masking.

\textbf{Self-Contrast Mechanism}\quad Given the current sequence state $X_t$ and the negative sample $\tilde{X}_t$, 
we perform two forward passes through $f_\theta$ to obtain:
\begin{equation}
    L_{\text{cond}} = f_\theta(X_t), \quad L_{\text{uncond}} = f_\theta(\tilde{X}_t)
\end{equation}
The self-contrastive logits $\hat{L}_t$ are then formulated as:
\begin{equation}
    \hat{L}_t = L_{\text{uncond}} + (\omega + 1) \cdot \big( L_{\text{cond}} - L_{\text{uncond}} \big)
\end{equation}
where $\omega > 0$ is the guidance scale, which is mathematically equivalent 
to $\hat{L}_t = L_{\text{cond}} + \omega \cdot \big( L_{\text{cond}} - L_{\text{uncond}} \big)$. The resulting $\hat{L}_t$ subsequently replaces the original logits 
in the decoding step, where tokens are sampled based on the 
self-contrastively guided distribution.

\textbf{FoCore\_A variants}\quad We further introduce \textbf{FoCore\_A}, an efficient variant that accelerates inference via parallel decoding. Specifically, FoCore\_A continuously monitors the mean 
EMA-based instability score over undecoded positions, defined as 
$\bar{s}_i = \frac{1}{|\mathcal{U}_t|}\sum_{j \in \mathcal{U}_t} 
\mathrm{EMA}_t^{(j)}$. Once $\bar{s}_i < \tau$, indicating sufficient 
convergence of HD tokens within the current block, FoCore\_A immediately 
decodes $m$ additional high-confidence tokens in parallel at the current 
step, significantly reducing the total number of forward passes while 
preserving generation quality.

\textbf{Pseudo-Code}\quad The pseudo-code for FoCore and FoCore\_A is provided in Appendix \ref{app:pseudo}.

\textbf{Theoretical Analysis}\quad Appendix \ref{appendix:theory} offers a theoretical analysis of how FoCore facilitates decoding.

\section{Experiments}
\label{sec:experiments}
\subsection{Experimental Setup}
\label{sec:experiments_setup}
\textbf{Datasets and Backbones} To comprehensively evaluate the proposed approach, we conduct experiments across six widely adopted benchmarks covering mathematical reasoning, code generation, and logical reasoning tasks. Specifically, for mathematical reasoning, we utilize GSM8K~\citep{cobbe2021training} and MATH500~\citep{hendrycks2021measuring}. For code generation, we employ HumanEval~\citep{chen2021evaluating} and MBPP ~\citep{austin2021program}. For logical reasoning, we evaluate on SVAMP~\citep{patel2021nlp} and Countdown~\citep{gandhi2024stream}. All experiments are conducted using two robust instruction-tuned large language models as our backbones: LLaDA 8B Instruct~\citep{nie2025large} and Dream 7B Instruct~\citep{ye2025dream}. Further implementation details are provided in Appendix~\ref{app:implementation}.

\textbf{Baselines} We benchmark our proposed method against a diverse set of strong baselines to demonstrate its effectiveness. These baselines can be broadly categorized into two groups. The first group consists of standard and widely used decoding strategies, including Random sampling, Greedy decoding, Confidence-based decoding, and Standard Classifier-Free Guidance (Std CFG). The second group comprises recent state-of-the-art (SOTA) advanced decoding methods, specifically A-CFG~\citep{li2025adaptive}, KLASS ~\citep{kimklass}, and Prophet~\citep{li2025diffusion}.

\subsection{Main Results}
\textbf{Effectiveness.}\quad
As illustrated in Table~\ref{tab:main_results}, FoCore achieves systematic performance gains over all baselines on both Math and Code tasks under the LLaDA-8B-Instruct backbone. On mathematical benchmarks, FoCore attains 75.05 on GSM8K and 43.80 on MATH500, surpassing the strongest baselines KLASS (74.14) and Confidence (42.60), respectively. The improvements are particularly pronounced on code generation: FoCore yields gains of +3.1 to +3.7 percentage points over the strongest baselines on HumanEval, and consistently ranks first across all generation lengths on MBPP (38.40/40.40/40.00), uniformly outperforming all competitors including Prophet. As a lightweight variant, FoCore\_A incurs only marginal performance degradation while remaining competitive across most configurations. Results on SVAMP and Countdown are reported in Appendix~\ref{app:experiment}.

\begin{table*}[h!]
\centering
\small
\setlength{\tabcolsep}{1.0mm}
\resizebox{\linewidth}{!}{
\begin{tabular}{l | ccc | ccc | ccc | ccc}
\toprule
\multicolumn{1}{r|}{Dataset} &
\multicolumn{3}{c|}{GSM8K (Math)} &
\multicolumn{3}{c|}{MATH500 (Math)} &
\multicolumn{3}{c|}{HumanEval (Code)} &
\multicolumn{3}{c}{MBPP (Code)}\\
\cmidrule(lr){2-4} \cmidrule(lr){5-7} \cmidrule(lr){8-10} \cmidrule(lr){11-13}
\multicolumn{1}{r|}{Methods / Len} & 128 & 256 & 512 & 128 & 256 & 512 & 128 & 256 & 512 & 128 & 256 & 512 \\
\midrule
\multicolumn{13}{c}{\textbf{LLaDA-8B-Instruct}} \\
\midrule
Random       & 63.76 & 67.93 & 69.14 & 26.20 & 28.60 & 27.00 & 18.29 & 23.17 & 22.56 & 26.60 & 28.20 & 27.00 \\
Greedy       & 74.07 & 72.70 & 74.07 & 32.20 & 42.40 & \colorbox{green!8}{39.20} & 27.43 & 37.80 & 36.58 & 36.80 & 36.20 & 36.80 \\
Confidence   & 74.07 & 71.03 & \colorbox{green!8}{74.07} & 32.80 & \colorbox{green!8}{42.60} & 38.60 & 27.43 & 37.80 & 36.58 & \colorbox{green!8}{37.40} & 36.60 & 37.20 \\
Std CFG      & 71.64 & 71.72 & 73.16 & \colorbox{green!8}{34.20} & 33.40 & 35.60 & \colorbox{green!8}{33.53} & 39.02 & 38.41 & 36.20 & 39.40 & 38.40 \\
A-CFG \textcolor{blue}{\scriptsize [NeurIPS 25]} & 71.11 & 70.43 & 72.32 & 33.20 & 38.20 & 37.60 & 26.82 & 31.09 & \colorbox{green!8}{41.46} & 34.40 & 36.00 & 37.40 \\
KLASS \textcolor{blue}{\scriptsize [NeurIPS 25]} & \colorbox{green!8}{74.14} & 71.19 & 74.37 & 32.80 & 42.20 & 38.60 & 27.43 & 37.80 & 37.19 & 37.00 & 37.00 & 37.20 \\
Prophet \textcolor{blue}{\scriptsize [ICLR 26]}  & -- & 69.44 & -- & -- & 31.20 & -- & -- & 36.58 & -- & -- & \colorbox{green!8}{39.80} & -- \\
\midrule
\textbf{FoCore (ours)}   & \colorbox{red!8}{\textbf{74.67}} & \colorbox{red!8}{\textbf{75.05}} & \colorbox{green!8}{74.45} & 32.80 & \colorbox{red!8}{\textbf{43.80}} & 39.00 & \colorbox{red!8}{\textbf{36.58}} & \colorbox{red!8}{\textbf{42.68}} & \colorbox{red!8}{\textbf{44.51}} & \colorbox{red!8}{\textbf{38.40}} & \colorbox{red!8}{\textbf{40.40}} & \colorbox{red!8}{\textbf{40.00}} \\
\textbf{FoCore\_A (ours)} & 73.46 & \colorbox{green!8}{73.99} & \colorbox{red!8}{\textbf{75.44}} & \colorbox{red!8}{\textbf{34.80}} & 42.20 & \colorbox{red!8}{\textbf{39.60}} & 28.04 & \colorbox{green!8}{40.24} & 40.85 & 36.60 & 36.20 & \colorbox{green!8}{39.80} \\
\midrule

\multicolumn{13}{c}{\textbf{Dream-Instruct-7B}} \\
\midrule
Random       & 35.10 & 37.60 & 38.13 & 14.20 & 14.00 & 16.80 & 18.90 & 22.50 & 21.30 & 29.40 & 30.80 & 32.40 \\
Greedy       & 75.11 & \colorbox{green!8}{80.66} & 79.22 & 40.00 & 38.00 & 42.20 & 52.43 & 53.65 & 50.00 & 55.20 & 56.80 & 57.00 \\
Confidence   & \colorbox{green!8}{76.19} & 79.52 & 79.22 & 41.00 & 39.80 & 43.40 & 56.09 & 53.65 & 49.39 & 55.20 & 56.60 & 57.40 \\
Std CFG      & 75.13 & 80.06 & 79.45 & \colorbox{green!8}{40.60} & 40.00 & 42.00 & \colorbox{green!8}{\textbf{56.70}} & 52.44 & 52.43 & 54.60 & 58.00 & 57.20 \\
A-CFG \textcolor{blue}{\scriptsize [NeurIPS 25]} & 73.84 & 80.21 & 78.01 & 40.40 & \colorbox{green!8}{42.80} & 40.60 & 53.04 & 51.82 & 48.17 & 56.20 & 57.80 & 57.40 \\
KLASS \textcolor{blue}{\scriptsize [NeurIPS 25]} & 76.04 & 79.37 & \colorbox{green!8}{80.13} & 39.40 & 39.00 & \colorbox{green!8}{43.80} & 53.66 & 51.21 & 52.43 & 54.40 & 56.80 & \colorbox{green!8}{57.60} \\
Prophet \textcolor{blue}{\scriptsize [ICLR 26]}  & -- & 72.25 & -- & -- & 39.60 & -- & -- & 51.80 & -- & -- & 57.60 & -- \\
\midrule
\textbf{FoCore (ours)}   & 75.43 & \colorbox{red!8}{\textbf{80.97}} & \colorbox{red!8}{\textbf{81.11}} & \colorbox{red!8}{\textbf{43.40}} & \colorbox{red!8}{\textbf{43.80}} & \colorbox{red!8}{\textbf{44.40}} & \colorbox{red!8}{\textbf{56.70}} & \colorbox{red!8}{\textbf{54.87}} & \colorbox{red!8}{\textbf{53.04}} & \colorbox{green!8}{\textbf{56.60}} & \colorbox{red!8}{\textbf{58.20}} & \colorbox{red!8}{\textbf{59.80}} \\
\textbf{FoCore\_A (ours)} & \colorbox{red!8}{\textbf{76.64}} & 80.59 & 79.15 & 40.40 & 41.80 & \colorbox{green!8}{\textbf{44.40}} & 54.87 & \colorbox{green!8}{53.65} & \colorbox{green!8}{53.04} & \colorbox{red!8}{\textbf{56.60}} & \colorbox{green!8}{58.00} & 56.20 \\
\bottomrule
\end{tabular}
}
\caption{
Quantitative comparison of FoCore and FoCore\_A against baselines on math (GSM8K, MATH500) and code (HumanEval, MBPP) benchmarks under the LLaDA-8B-Instruct and Dream-Instruct-7B backbones, with generation lengths $\{128, 256, 512\}$. For each column, the best result within each backbone block is highlighted in \colorbox{red!8}{red} and the second-best in \colorbox{green!8}{green}; the column-wise maximum is additionally marked in \textbf{bold}.
}
\vspace{-2em}   
\label{tab:main_results}
\end{table*}

\textbf{Scalability.}\quad
FoCore exhibits a stable and monotonically increasing performance trend as the generation length scales from 128 to 512, in contrast to the unstable scaling behavior observed in most baselines. On HumanEval, FoCore consistently improves across lengths (36.58$\to$42.68$\to$44.51), making it the only method that outperforms all baselines at every length; a similar monotonic trend holds on GSM8K (74.67$\to$75.05), while Greedy and Confidence suffer a noticeable drop at length 256 (72.70), and A-CFG only reaches its peak at length 512 (41.46). These results collectively demonstrate that FoCore makes more effective use of longer generation budgets. FoCore\_A likewise exhibits robust scalability, achieving the column-best score of 75.44 on GSM8K at length 512, with only marginal underperformance relative to FoCore at certain intermediate lengths.

\textbf{Generalizability.}
FoCore maintains its dominant performance when transferred to the Dream backbone, demonstrating strong generalizability across discrete diffusion language model architectures. On MATH500, FoCore attains column-best results across all three generation lengths (43.40/43.80/44.40), comprehensively surpassing all baselines. On code generation, FoCore achieves 59.80 on MBPP at length 512, outperforming the strongest baseline KLASS (57.60) by +2.2 percentage points, and consistently outperforms or matches all baselines on HumanEval. FoCore\_A similarly suffers only negligible performance degradation, collectively demonstrating the robust generalizability of the proposed framework across the family of discrete diffusion language models.
\begin{wraptable}{r}{0.7\linewidth}   
\centering
\small
\setlength{\tabcolsep}{1.0mm}
\vspace{-1em}   
\resizebox{\linewidth}{!}{\begin{tabular}{l | cc | cc | cc | cc}
\toprule
\multicolumn{1}{r|}{Dataset}
  & \multicolumn{2}{c|}{GSM8K}
  & \multicolumn{2}{c|}{MATH500}
  & \multicolumn{2}{c|}{HumanEval}
  & \multicolumn{2}{c}{MBPP} \\
\midrule
\multicolumn{9}{c}{\textbf{(a) Decoding Step Acceleration}} \\
\cmidrule(lr){2-3}\cmidrule(lr){4-5}\cmidrule(lr){6-7}\cmidrule(lr){8-9}
\multicolumn{1}{r|}{Method}
  & Perf.$\uparrow$ & Speedup$\uparrow$
  & Perf.$\uparrow$ & Speedup$\uparrow$
  & Perf.$\uparrow$ & Speedup$\uparrow$
  & Perf.$\uparrow$ & Speedup$\uparrow$ \\
\midrule
Prophet
  & 69.44 & 1.33$\times$
  & 31.20 & 1.25$\times$
  & 36.58 & 1.26$\times$
  & \textbf{39.80} & 1.24$\times$ \\
KLASS
  & 71.19 & \textbf{1.90$\times$}
  & 42.20 & \textbf{2.13$\times$}
  & 37.80 & \textbf{2.42$\times$}
  & 37.00 & 1.45$\times$ \\
\textbf{FoCore\_A}
  & \textbf{73.99} & 1.45$\times$
  & \textbf{42.20} & 1.88$\times$
  & \textbf{40.24} & 2.07$\times$
  & 36.20 & \textbf{2.29$\times$} \\
\midrule
\multicolumn{9}{c}{\textbf{(b) End-to-End Generation Efficiency}} \\
\cmidrule(lr){2-3}\cmidrule(lr){4-5}\cmidrule(lr){6-7}\cmidrule(lr){8-9}
\multicolumn{1}{r|}{Method}
  & Perf.$\uparrow$ & Time(s)$\downarrow$
  & Perf.$\uparrow$ & Time(s)$\downarrow$
  & Perf.$\uparrow$ & Time(s)$\downarrow$
  & Perf.$\uparrow$ & Time(s)$\downarrow$ \\
\midrule
A-CFG
  & 70.43 & 19.85
  & 38.20 & 21.38
  & 31.09 & 22.56
  & 36.00 & 19.27 \\
CFG
  & 71.72 & 16.61
  & 33.40 & 18.28
  & 39.02 & 20.76
  & \textbf{39.40} & 15.49 \\
\textbf{FoCore\_A}
  & \textbf{73.99} & \textbf{12.65}
  & \textbf{42.20} & \textbf{11.56}
  & \textbf{40.24} & \textbf{8.64}
  & 36.20 & \textbf{7.45} \\
\bottomrule
\end{tabular}}
\caption{
  Efficiency comparison across four benchmarks.
  \textbf{(a)} Decoding step speedup ($\times$) over standard decoding,
  paired with task performance (\%).
  \textbf{(b)} End-to-end generation time (s), paired with task performance (\%).
  $\uparrow$ higher is better; $\downarrow$ lower is better.
  Bold denotes best per column.
}
\label{tab:efficiency}
\vspace{-1em}   
\end{wraptable}

\textbf{Computational Complexity.}\quad Beyond performance, FoCore\_A achieves favorable efficiency by leveraging 
EMA-based instability monitoring to trigger parallel decoding, which 
significantly reduces the total number of forward passes. As shown in 
Table~\ref{tab:efficiency}(a), FoCore\_A attains 1.45$\times$ to 
2.29$\times$ decoding step speedup across all benchmarks, while 
simultaneously outperforming competing efficient methods Prophet and KLASS 
in task performance on GSM8K (73.99 vs.\ 71.19) and HumanEval (40.24 
vs.\ 37.80). In terms of end-to-end generation time 
(Table~\ref{tab:efficiency}(b)), FoCore\_A achieves the lowest wall-clock 
time across nearly all benchmarks, reducing generation time by up to 
$2.4\times$ compared to A-CFG and $1.3\times$ compared to Std CFG, while 
maintaining superior task performance. These results demonstrate that 
FoCore\_A strikes a favorable trade-off between efficiency and 
effectiveness, making it a practical choice for inference under 
computational constraints. A detailed ablation study on the parallel decoding hyperparameter $m$ 
is provided in Appendix~\ref{app:ablation_m}.

\subsection{Ablation Study and Analysis}
\textbf{Effect of $\omega$}\quad As illustrated in Figure~\ref{fig:sen}(a), the model demonstrates strong robustness to guidance scale across both datasets, maintaining stable performance over a broad range of values. The guidance scale fundamentally governs the trade-off between exploration and exploitation during the search process: an excessively small scale provides insufficient guidance signals, hindering the model from effectively concentrating on high-quality candidates, while an overly large scale over-constrains the search space diversity and risks converging to local optima. It is also noteworthy that performance on the Countdown task exhibits greater sensitivity to this hyperparameter than on SVAMP, which aligns with its nature as a combinatorial search task that inherently demands a wider exploration space. Overall, a guidance scale of approximately 0.3 is recommended, as it strikes an optimal balance between guidance strength and search diversity.
\begin{wrapfigure}{r}{0.6\linewidth}
    \centering
    \vspace{-1em}
    \includegraphics[width=\linewidth]{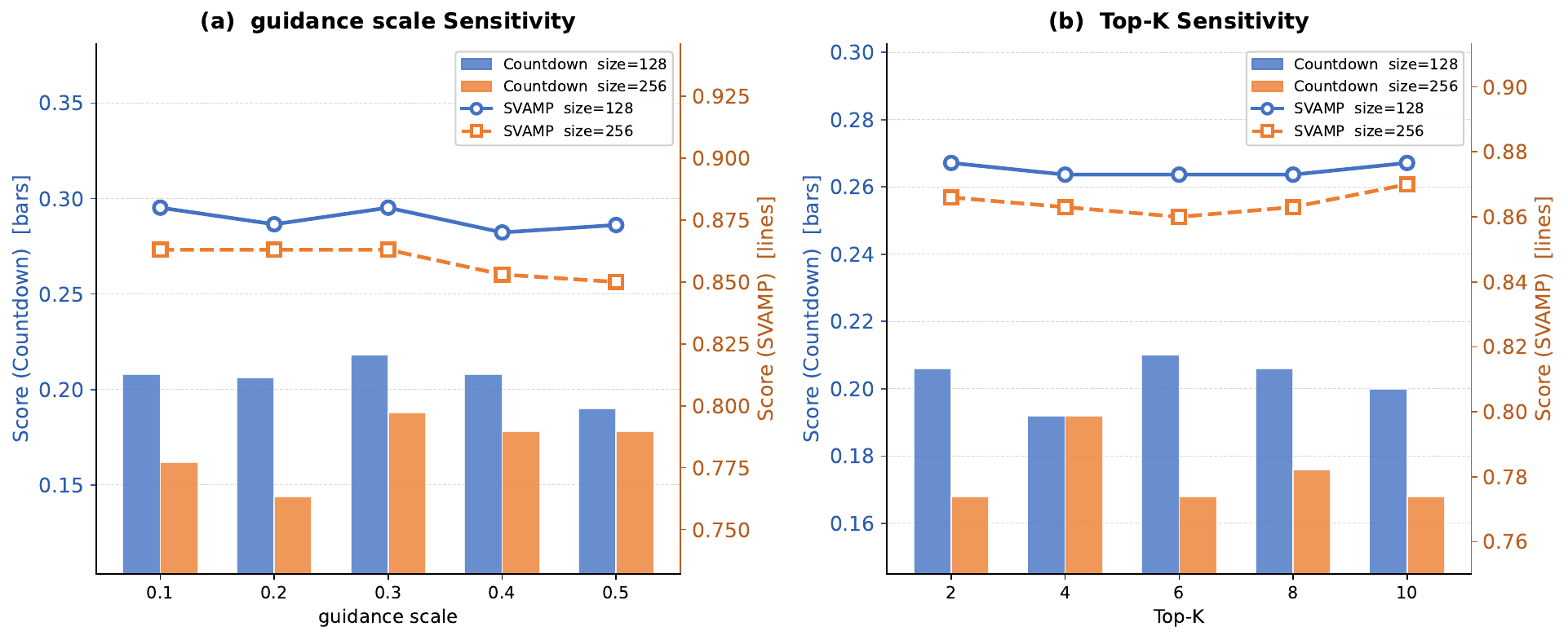}
    \caption{Hyperparameter sensitivity analysis of FoCore with respect to two key hyperparameters 
$\omega$ and Top-$K$ on Countdown and SVAMP benchmarks at generation 
lengths 128 and 256.
\textbf{(a)} Guidance scale $\omega \in \{0.1, 0.2, 0.3, 0.4, 0.5\}$.
\textbf{(b)} Top-$K \in \{2, 4, 6, 8, 10\}$.}
    \vspace{-1em}
    \label{fig:sen}
\end{wrapfigure}

\textbf{Effect of Top-K}\quad As illustrated in Figure~\ref{fig:sen}(b), the model exhibits minimal performance fluctuation across both datasets with respect to Top-K, demonstrating considerable robustness to this 
hyperparameter. This observation implies that HD Tokens are inherently sparse during the decoding process, such that even at relatively small values of K, the critical candidates are sufficiently covered and the model is capable of sustaining stable performance. However, as K grows substantially, the incorporation of excessive noisy candidates results in marginal performance degradation, suggesting that indiscriminately enlarging the candidate set does not consistently yield improvements. Furthermore, the optimal choice of K is found to interact with task-specific characteristics: for tasks characterized by a more open solution space, moderately increasing K promotes broader candidate coverage and proves beneficial, whereas for tasks with a more concentrated solution space, a smaller K already suffices to capture the most essential candidates, and an overly large candidate pool instead introduces unwanted interference.

\textbf{Case Study of FoCore}\quad In Appendix \ref{app:case}, we provide a qualitative analysis of FoCore, wherein the HD tokens during the decoding process are highlighted to demonstrate how FoCore leverages these anchors to guide the model toward the correct path.
\section{Conclusion}
\label{sec:conclusion}
This work uncovers a critical yet long-overlooked property of DLMs: the heterogeneous information density distribution inherent in the generated context. Through systematic investigation of high-information-density (HD) tokens, we demonstrate their central role in both semantic guidance and decoding acceleration. Specifically, FoCore steers generation by exploiting HD tokens in a training-free, self-contrastive manner, while FoCore\_A further leverages their convergence behavior to substantially reduce decoding steps and per-sample latency. Extensive experiments across multiple benchmarks validate the effectiveness and efficiency of the proposed framework. Looking forward, we believe that incorporating HD token awareness into RL-based post-training of DLMs represents a promising direction for internalizing the focus-on-core principle at the parameter level.



\bibliographystyle{plainnat}
\bibliography{reference}

\newpage
\appendix
\section{Related Works}
\label{sec:related_work}
\textbf{Diffusion Language Models}\quad While autoregressive (AR) models dominate natural language generation, their strict left-to-right causal nature restricts bidirectional context modeling and suffers from error propagation. Diffusion Language Models (DLMs) have emerged as a compelling alternative, evolving from early continuous formulations~\citep{han2023ssd,dieleman2022continuous} to discrete Masked Diffusion Models~\citep{saharia2022photorealistic,austin2021structured}. Recently, large-scale DLMs such as LLaDA~\citep{nie2025large} and Dream~\citep{ye2025dream} have demonstrated generative and reasoning performance on par with canonical AR baselines.

Fundamentally, DLMs diverge from AR models in their decoding paradigm. Instead of sequential token prediction, DLMs formulate generation as an iterative and bidirectional mask-infilling process~\citep{ye2024diffusion,li2025survey}. By globally refining a fully masked sequence over discrete timesteps, DLMs naturally support parallel decoding. However, this introduces a severe efficiency-accuracy trade-off: unmasking multiple tokens simultaneously accelerates inference but often degrades semantic quality due to the loss of joint dependency modeling~\citep{hu2025accelerating,wu2025fast}. While existing literature predominantly focuses on architectural scaling or training objectives~\citep{gong2024scaling,bie2025llada2}, our work explicitly targets the underexplored bottleneck of high-quality decoding to fully unleash the intrinsic reasoning capabilities of DLMs during the iterative generation phase.

\textbf{Test-Time Decoding and Guidance}\quad Recent studies on DLMs extensively explore test-time decoding strategies to optimize generation quality, diversity, and computational efficiency. To address the critical challenge of token commitment timing, researchers have theoretically validated the efficacy of confidence-based decoding~\citep{cai2026confidence}. Practical methods, such as KLASS~\citep{kimklass} and Prophet~\citep{li2025diffusion}, utilize entropy bounds or KL-divergence to dictate unmasking schedules, while hybrid approaches dynamically alternate between beam search and rapid sampling to balance exploration and computational overhead~\citep{cao2026search}. Capitalizing on the non-autoregressive nature of DLMs, frameworks including Dynamic Sliding Block scheduling~\cite{luo2026dsb} and SlowFast sampling~\citep{wei2025accelerating} modulate decoding step sizes to accelerate parallel generation without compromising output quality. Advancing inference-time scaling and reasoning, techniques like trajectory refinement~\citep{dang2025inference} and time-annealed perturbation sampling~\citep{wu2026time} significantly boost generation fidelity. Additionally, to mitigate the myopic assumptions of standard greedy decoding, lookahead mechanisms~\citep{mo2025decoding} are introduced to foresee future states and guide current decisions. Finally, reliability-enhancing paradigms from autoregressive models, such as Self-Consistency and semantic entropy~\citep{wang2025time}, have been successfully adapted to DLMs. Integrated with test-time heuristics like inference-time remasking~\citep{huang2025don}, these strategies effectively reduce hallucinations and improve the temporal stability of generated outputs.

Beyond optimizing decoding schedules and unmasking heuristics, inference-time guidance has emerged as a pivotal direction for DLMs, enabling generation steering toward desired attributes without retraining. Originating from continuous diffusion models~\citep{ho2022classifier} and subsequently adapted to continuous language embeddings~\citep{li2022diffusion}, Classifier-Free Guidance (CFG) and its variants have recently been extensively explored in discrete DLM settings. Training-free paradigms have rapidly evolved, encompassing contrastive guided decoding via learned classifiers~\citep{lovelace2024diffusion}, posterior sampling for hard lexical constraints~\citep{cardeiconstrained}, and iterative logit refinement~\citep{avrahami2026ilrr}. To ensure faithful sampling, techniques like Sequential Monte Carlo debiasing~\citep{lee2025debiasing} have been introduced to address the bias of approximate guidance in discrete spaces, while Reward-Free Guidance~\citep{chen2025rfg} leverages implicit value estimates from the denoising model itself to enhance reasoning without external reward models.

Despite these advancements, standard CFG in discrete DLMs typically relies on interpolating conditional predictions with a static unconditional baseline (e.g., a fully masked sequence), which often depends on superficial confidence scores and fails to adapt to dynamic certainty shifts across token positions during iterative refinement. Recent efforts to address this limitation include adaptive guidance strategies that modulate CFG strength based on per-token uncertainty estimates~\citep{li2025adaptive}, yet these approaches remain constrained by the fundamental paradigm of contrasting conditional against a fixed unconditioned state. Addressing this critical gap, we propose a novel approach that leverages a core reasoning unit to provide dynamic, step-aware guidance throughout the denoising trajectory.
\section{More Implementation Details}
\label{app:implementation}
\subsection{Datasets}
We provide detailed descriptions of the datasets used in our experiments as follows:

\begin{itemize}
    \item \textbf{GSM8K}~\citep{cobbe2021training} is a benchmark consisting of 8,500 
    high-quality grade school math word problems, requiring multi-step arithmetic reasoning 
    to solve. Each problem is annotated with a detailed natural language solution, making it 
    widely adopted for evaluating the mathematical reasoning capabilities of language models.

    \item \textbf{MATH500}~\citep{hendrycks2021measuring} is a challenging mathematics 
    benchmark sampled from the MATH dataset, spanning 500 competition-level problems across 
    seven subject areas including algebra, geometry, and number theory. It is designed to 
    assess advanced mathematical problem-solving ability and is considered significantly more 
    difficult than GSM8K.

    \item \textbf{HumanEval}~\citep{chen2021evaluating} is a code generation benchmark 
    comprising 164 hand-crafted Python programming problems, each accompanied by a function 
    signature, docstring, and a set of unit tests. Model-generated solutions are evaluated 
    via functional correctness through test case execution.

    \item \textbf{MBPP}~\citep{austin2021program} (Mostly Basic Python Problems) consists 
    of approximately 500 crowd-sourced Python programming tasks designed for entry-level 
    programmers. Each problem includes a natural language description, a reference solution, 
    and automated test cases for evaluation.

    \item \textbf{SVAMP}~\citep{patel2021nlp} is a challenge dataset for arithmetic word 
    problems, constructed by applying controlled structural variations to existing problems. 
    It is specifically designed to test the robustness of models against simple textual 
    perturbations that should not affect the underlying mathematical reasoning.

    \item \textbf{Countdown}~\citep{gandhi2024stream} is a combinatorial arithmetic 
    reasoning task in which the model is required to combine a given set of numbers using 
    basic arithmetic operations to reach a specified target value. Its combinatorially large 
    search space makes it a particularly demanding benchmark for evaluating structured 
    reasoning and search capabilities.
\end{itemize}
\begin{algorithm}[htbp]
\caption{FoCore: Focus on the Core}
\label{alg:focore}
\begin{algorithmic}[1]
\Require Model $f_\theta$, prompt $X_0$, generation length $L$, block length $B$,
         steps per block $T$, guidance scale $\omega$, Top-$K$, EMA decay $\alpha$
\Ensure Generated sequence $X$

\State Initialize $X \leftarrow [X_0; \texttt{[MASK]}^L]$
\State Initialize $S^{(i)} \leftarrow 0$, $P_{\text{prev}}^{(i)} \leftarrow \textbf{0}$ for all positions $i$

\For{each block $b = 1, \ldots, \lfloor L/B \rfloor$}
    \For{each step $t = 1, \ldots, T$}
        \State $\mathcal{U}_t \leftarrow \{i : X_i \neq \texttt{[MASK]}\}$ \Comment{Unmasked token indices}

        \State \textcolor{gray}{\textit{// Step 1: Compute conditional logits and distributions}}
        \State $L_{\text{cond}} \leftarrow f_\theta(X_t)$
        \State $P_t(x_i) \leftarrow \mathrm{Softmax}(L_{\text{cond}}[i])$ for all $i$

        \State \textcolor{gray}{\textit{// Step 2: Compute Top-}$K$ \textit{truncated JS divergence}}
        \For{each position $i \in \mathcal{U}_t$}
            \State $\mathcal{V}_K \leftarrow \mathrm{Top\text{-}}K\big(P_t(x_i)\big)$
            \State $M(v) \leftarrow \frac{1}{2}\big(P_t(v) + P_{t-1}(v)\big)$, $\forall v \in \mathcal{V}_K$
            \State $D_t^{(i)} \leftarrow \frac{1}{2}\sum_{v \in \mathcal{V}_K} P_t(v)\log\frac{P_t(v)}{M(v)} + \frac{1}{2}\sum_{v \in \mathcal{V}_K} P_{t-1}(v)\log\frac{P_{t-1}(v)}{M(v)}$
        \EndFor

        \State \textcolor{gray}{\textit{// Step 3: Update EMA instability scores}}
        \For{each position $i \in \mathcal{U}_t$}
            \State $S_t^{(i)} \leftarrow \alpha \cdot S_{t-1}^{(i)} + (1 - \alpha) \cdot D_t^{(i)}$
        \EndFor

        \State \textcolor{gray}{\textit{// Step 4: Identify HD tokens and construct negative sample}}
        \State $\epsilon \sim \mathcal{U}(0, 10^{-8})$ \Comment{Tie-breaking noise}
        \State $\mathcal{K}_t \leftarrow \mathop{\arg\max}_{\mathcal{K} \subset \mathcal{U}_t, |\mathcal{K}|=K} \sum_{j \in \mathcal{K}} \left(S_t^{(j)} + \epsilon\right)$
        \State $(\tilde{X}_t)_j \leftarrow \texttt{[MASK]}$ if $j \in \mathcal{K}_t$, else $(\tilde{X}_t)_j \leftarrow (X_t)_j$

        \State \textcolor{gray}{\textit{// Step 5: Self-Contrast}}
        \State $L_{\text{uncond}} \leftarrow f_\theta(\tilde{X}_t)$
        \State $\hat{L}_t \leftarrow L_{\text{uncond}} + (\omega + 1) \cdot \big(L_{\text{cond}} - L_{\text{uncond}}\big)$

        \State \textcolor{gray}{\textit{// Step 6: Decode tokens based on guided logits}}
        \State Update $X_t$ by unmasking high-confidence tokens via $\hat{L}_t$
        \State $P_{\text{prev}} \leftarrow P_t$ \Comment{Cache distributions for next step}
    \EndFor
\EndFor
\State \Return $X$
\end{algorithmic}
\end{algorithm}

\begin{algorithm}[htbp]
\caption{FoCore\_A: FoCore: Focus on the Core Acceleration}
\label{alg:focore_a}
\begin{algorithmic}[1]
\Require Model $f_\theta$, prompt $X_0$, generation length $L$, block length $B$,
         steps per block $T$, guidance scale $\omega$, Top-$K$, EMA decay $\alpha$,
         early-exit threshold $\tau$, parallel decoding count $m$
\Ensure Generated sequence $X$

\State Initialize $X \leftarrow [X_0; \texttt{[MASK]}^L]$
\State Initialize $S^{(i)} \leftarrow 0$, $P_{\text{prev}}^{(i)} \leftarrow \mathbf{0}$ for all positions $i$

\For{each block $b = 1, \ldots, \lfloor L/B \rfloor$}
    \State $\texttt{early\_exit} \leftarrow \texttt{False}$
    \For{each step $t = 1, \ldots, T$}
        \If{$\texttt{early\_exit}$}
            \State \textbf{break}
        \EndIf

        \State $\mathcal{U}_t \leftarrow \{i : X_i \neq \texttt{[MASK]}\}$

        \State \textcolor{gray}{\textit{// Steps 1--5: Same as FoCore (Algorithm~\ref{alg:focore})}}
        \State Compute $L_{\text{cond}}$, $P_t$, $D_t^{(i)}$, $S_t^{(i)}$, $\mathcal{K}_t$, $\tilde{X}_t$, $\hat{L}_t$

        \State \textcolor{gray}{\textit{// Step 6: Monitor mean instability for early-exit}}
        \State $\bar{s}_t \leftarrow \frac{1}{|\mathcal{U}_t|} \sum_{j \in \mathcal{U}_t} S_t^{(j)}$

        \If{$\bar{s}_t < \tau$} \Comment{HD tokens have sufficiently converged}
            \State Decode $m$ additional high-confidence tokens in parallel via $\hat{L}_t$
            \State $\texttt{early\_exit} \leftarrow \texttt{True}$
        \Else
            \State \textcolor{gray}{\textit{// Normal decoding step}}
            \State Update $X_t$ by unmasking high-confidence tokens via $\hat{L}_t$
        \EndIf

        \State $P_{\text{prev}} \leftarrow P_t$
    \EndFor
\EndFor
\State \Return $X$
\end{algorithmic}
\end{algorithm}
\subsection{Implementation Details}
\label{app:implementation_details_exp}
Our implementation builds upon the official codebases of 
LLaDA\footnote{\url{https://github.com/ML-GSAI/LLaDA}} and 
Dream\footnote{\url{https://github.com/DreamLM/Dream}}, 
and all evaluations are conducted using the \texttt{lm-evaluation-harness} 
framework\footnote{\url{https://github.com/EleutherAI/lm-evaluation-harness}} 
to ensure consistency and reproducibility across benchmarks. 
While this framework provides variance metrics for the evaluations, we empirically observed that the variance remains approximately consistent across different methods; therefore, we report only the mean performance scores for clarity. 
All experiments are performed on NVIDIA H800 GPUs (80GB memory), with an average inference latency of approximately 10 seconds per sample.

\paragraph{Evaluation Tasks and Settings.}
We evaluate our method on six benchmarks spanning three domains: 
mathematical reasoning (GSM8K and MATH500), code generation (HumanEval and MBPP), 
and logical reasoning (SVAMP and Countdown). 
For GSM8K and MATH500, we adopt a zero-shot chain-of-thought prompting strategy. 
For HumanEval and MBPP, model performance is assessed via functional correctness 
through automated test case execution.

\paragraph{Hyperparameter Settings.}
For all main experiments on LLaDA, the sampling temperature is set to 0, and the \texttt{low\_confidence} remasking strategy is employed. We set the EMA decay factor to 0.9 and use $k = 256$ for the Top-$K$ approximation in JS divergence computation. By default, the block length is 32, the guidance scale (CFG scale) is 0.3, and the Top-$K$ sampling value is fixed at 8 across all tasks. Generation lengths are configured per benchmark, with the number of denoising steps strictly matching the generation length. For FoCore\_A, the early-exit threshold is set to $\tau = 0.01$, and the number of parallel decoding tokens, $m$, ranges from 1 to 10, tailored to the specific characteristics of each benchmark.

\paragraph{Baselines.}
We benchmark our proposed method against a diverse set of strong baselines. 
The first group consists of standard decoding strategies, 
including Random Sampling, Greedy Decoding, Confidence-based Decoding, 
and Standard Classifier-Free Guidance (Std CFG). 
The second group comprises recent state-of-the-art decoding methods, 
specifically A-CFG~\citep{li2025adaptive}, KLASS~\citep{kimklass}, 
and Prophet~\citep{li2025diffusion}. 
All baselines are reproduced using their original hyperparameter configurations 
to ensure a fair comparison.
\subsection{License of models and data}
\label{app:license}
All models and datasets used in this work are publicly available for academic research. Their references and specific licenses are summarized below:

\paragraph{Models:}
\begin{itemize}
    \item \textbf{LLaDA}~\citep{bie2025llada2} : Apache License 2.0.
    \item \textbf{Dream}~\citep{ye2025dream} : Apache License 2.0.
\end{itemize}

\paragraph{Datasets:}
\begin{itemize}
    \item \textbf{GSM8K}~\citep{cobbe2021training}: MIT License.
    \item \textbf{MATH500}~\citep{hendrycks2021measuring}: MIT License.
    \item \textbf{HumanEval}~\citep{chen2021evaluating}: MIT License.
    \item \textbf{MBPP}~\citep{austin2021program}: CC-BY 4.0 License.
    \item \textbf{SVAMP}~\citep{patel2021nlp}: MIT License.
    \item \textbf{Countdown}~\citep{gandhi2024stream}: MIT License.
\end{itemize}

\section{Pseudo Code}
\label{app:pseudo}
We provide the pseudo code for FoCore (Algorithm~\ref{alg:focore}) and FoCore\_A (Algorithm~\ref{alg:focore_a}).

\section{Theoretical Analysis}
\label{appendix:theory}

\subsection{Preliminary: Notation and Setup}

Let $\mathcal{V}$ denote the vocabulary of size $|\mathcal{V}|$, and 
$\Delta^{|\mathcal{V}|-1}$ the corresponding probability simplex. For 
a sequence of length $N$, the predicted distribution at position $i$ 
and decoding step $t$ is denoted by $P_t^{(i)} \in \Delta^{|\mathcal{V}|-1}$, 
for $t = 0, 1, \ldots, T$.

\begin{definition}[Positional Conditional Entropy]
Given the ground-truth tokens at all other positions $x_{-i}^*$, the 
\emph{intrinsic semantic difficulty} of position $i$ is defined as the 
conditional entropy:
\begin{equation}
    \mathcal{H}_i \triangleq H(x_i \mid x_{-i}^*) 
    = -\sum_{v \in \mathcal{V}} p(x_i = v \mid x_{-i}^*) 
    \log p(x_i = v \mid x_{-i}^*)
\end{equation}
\end{definition}

\begin{definition}[Formal Characterization of HD/LD Tokens]
Given a threshold $\eta > 0$, we formally define the partition of 
tokens into hard-to-decode (HD) and easy-to-decode (LD) categories as:
\begin{equation}
    i \in \mathrm{HD} \iff \mathcal{H}_i > \eta; 
    \qquad 
    i \in \mathrm{LD} \iff \mathcal{H}_i \leq \eta
\end{equation}
\end{definition}

\subsection{Distribution Instability as a Proxy for Semantic Difficulty}
\label{appendix:theory:instability}

\subsubsection{Step-wise JS Divergence Upper-bounded by Conditional Entropy}

\begin{proposition}[Entropic Upper Bound on JS Divergence]
\label{prop:js_entropy_bound}
Let $f_\theta$ be a sufficiently trained masked diffusion language model 
satisfying $P_t^{(i)} \approx p(x_i \mid x^t)$. Let $\Delta^t$ denote 
the contextual increment introduced at step $t$, i.e., the set of 
newly decoded tokens. Then the step-wise JS divergence at position $i$ 
satisfies:
\begin{equation}
    D_t^{(i)} \leq \frac{1}{2} I(x_i;\, \Delta^t \mid x^{t-1}) 
    \leq \frac{1}{2} \mathcal{H}_i
\end{equation}
\end{proposition}

\begin{proof}
\textbf{First inequality.}
By the standard relationship between JS divergence and KL divergence:
\begin{equation}
    D_t^{(i)} = \mathrm{JS}(P_t^{(i)} \| P_{t-1}^{(i)}) 
    \leq \frac{1}{2} D_{\mathrm{KL}}(P_t^{(i)} \| P_{t-1}^{(i)})
\end{equation}
Since the distributional shift between $P_t^{(i)}$ and $P_{t-1}^{(i)}$ 
is induced solely by the contextual increment $\Delta^t$, applying the 
Data Processing Inequality yields:
\begin{equation}
    D_{\mathrm{KL}}(P_t^{(i)} \| P_{t-1}^{(i)}) 
    \leq I(x_i;\, \Delta^t \mid x^{t-1})
\end{equation}

\textbf{Second inequality.}
By the definition of conditional mutual information and the monotonicity 
of conditional entropy (conditioning on more information does not 
increase entropy):
\begin{equation}
    I(x_i;\, \Delta^t \mid x^{t-1}) 
    \leq H(x_i \mid x^{t-1}) 
    \leq H(x_i \mid x_{-i}^*) 
    = \mathcal{H}_i
\end{equation}
Combining the two inequalities completes the proof.
\end{proof}

\begin{corollary}[Inherent Stability of LD Tokens]
\label{cor:ld_stable}
If $i \in \mathrm{LD}$, then $\mathcal{H}_i \leq \eta$, and consequently:
\begin{equation}
    D_t^{(i)} \leq \frac{\eta}{2} \approx 0
\end{equation}
That is, the predictive distribution of an LD token remains stable 
under arbitrary contextual changes, with near-zero step-wise JS 
divergence. Conversely, for HD tokens ($\mathcal{H}_i > \eta$), the 
upper bound is looser, permitting larger step-wise divergence.
\end{corollary}

Proposition~\ref{prop:js_entropy_bound} and 
Corollary~\ref{cor:ld_stable} jointly provide a principled theoretical 
justification for the use of distributional instability as a proxy for 
identifying HD tokens: the observed distributional instability is a 
necessary consequence of high conditional entropy.

\subsubsection{Approximation Error of Top-\texorpdfstring{$K$}{K} Truncation}

FoCore employs a Top-$K$ truncated JS divergence to reduce computational 
overhead, restricting the computation to the $K$ highest-probability 
tokens under the current step's distribution. Let $\widetilde{D}_t^{(i)}$ 
denote the JS divergence computed over the full vocabulary, and 
$D_t^{(i)}$ its Top-$K$ truncated counterpart. Define the residual 
tail mass of the mixture distribution as:
\begin{equation}
    \epsilon_K = 1 - \sum_{v \in \mathcal{V}_K} M(v), 
    \quad 
    M(v) = \frac{1}{2}\big(P_t(v) + P_{t-1}(v)\big)
\end{equation}

\begin{proposition}[Truncation Approximation Error Bound]
\label{prop:topk_error}
\begin{equation}
    \left|\widetilde{D}_t^{(i)} - D_t^{(i)}\right| \leq \log 2 \cdot \epsilon_K
\end{equation}
\end{proposition}

\begin{proof}
By the boundedness of JS divergence, $0 \leq \mathrm{JS}(\cdot \| \cdot) 
\leq \log 2$. The contribution of the truncated token set 
$\mathcal{V} \setminus \mathcal{V}_K$ to the full JS divergence is 
upper bounded by:
\begin{equation}
    \log 2 \cdot \sum_{v \notin \mathcal{V}_K} M(v) = \log 2 \cdot \epsilon_K
\end{equation}
which directly yields the stated bound.
\end{proof}

\begin{remark}
In practice, the softmax output distributions of language models are 
highly concentrated over a small set of high-frequency tokens, 
resulting in $\epsilon_K \ll 1$. Consequently, the Top-$K$ truncation 
achieves an $O(|\mathcal{V}|/K)$-fold reduction in computational cost 
while introducing only an $O(\epsilon_K)$ approximation error, which 
is negligible in practice.
\end{remark}

\subsection{HD Tokens Dominate the Decoding Trajectory}
\label{appendix:theory:dominance}

\subsubsection{HD Tokens as Information Bottlenecks}

\begin{definition}[Positional Information Contribution]
The information contribution of position $i$ to the overall sequence 
is defined as its mutual information with all other positions:
\begin{equation}
    \mathcal{I}_i \triangleq I(x_i;\, x_{-i}^*) 
    = H(x_i) - H(x_i \mid x_{-i}^*) 
    = H(x_i) - \mathcal{H}_i
\end{equation}
\end{definition}

\begin{proposition}[HD Tokens as Information Bottlenecks]
\label{prop:hd_bottleneck}
For any subset of positions $\mathcal{S} \subset [N] \setminus \{i\}$, 
the conditional information contribution of position $i$ to the tokens 
in $\mathcal{S}$ satisfies:
\begin{equation}
    I(x_i;\, x_\mathcal{S}) \geq \mathcal{H}_\mathcal{S} 
    - \mathcal{H}_{\mathcal{S} \cup \{i\}}
\end{equation}
where $\mathcal{H}_\mathcal{S} = H(x_\mathcal{S} \mid x_{[N]\setminus \mathcal{S}})$ 
denotes the joint conditional entropy of the positions in $\mathcal{S}$.
\end{proposition}

\begin{proof}
By the chain rule of mutual information:
\begin{equation}
    I(x_i;\, x_\mathcal{S}) = H(x_\mathcal{S}) - H(x_\mathcal{S} \mid x_i)
\end{equation}
For an HD token at position $i$, the marginal entropy $H(x_i)$ is 
large (high prior uncertainty) and $\mathcal{H}_i > \eta$ (high 
intrinsic difficulty), implying that the realization of $x_i$ exerts 
a significant influence on the predictive uncertainty of other positions:
\begin{equation}
    H(x_\mathcal{S} \mid x_i = v_1) \neq H(x_\mathcal{S} \mid x_i = v_2), 
    \quad \forall\, v_1 \neq v_2
\end{equation}
By the convexity of conditional entropy and the non-negativity of 
mutual information, the reduction in uncertainty over $\mathcal{S}$ 
upon observing $x_i$, namely $\mathcal{H}_\mathcal{S} - 
\mathcal{H}_{\mathcal{S} \cup \{i\}}$, constitutes a lower bound on 
$I(x_i;\, x_\mathcal{S})$.
\end{proof}

Proposition~\ref{prop:hd_bottleneck} establishes that HD tokens are 
critical semantic constraint nodes in the sequence: their predictive 
uncertainty propagates through mutual information to all other positions, 
rendering the correct decoding of HD tokens an \emph{information 
bottleneck} that governs the overall generation quality.

\subsubsection{Variance Propagation: HD Tokens Induce Path-level Perturbation}

\begin{theorem}[Path-level Dominance of HD Tokens]
\label{thm:path_dominance}
Assume the denoising process of the masked diffusion model satisfies 
the Markov property, i.e., $P_t^{(j)} = [f_\theta(x^t)]_j$, and that 
$f_\theta$ is Lipschitz continuous with constant $C > 0$ with respect 
to its input. For position $j \neq i$, define the \emph{path sensitivity} 
of its predictive distribution to the token at position $i$ as:
\begin{equation}
    \Gamma_t^{(i \to j)} \triangleq 
    \mathbb{E}_{v_1, v_2 \sim P_t^{(i)}}\!\left[
    \left\|P_t^{(j)}\big|_{x_i = v_1} 
    - P_t^{(j)}\big|_{x_i = v_2}\right\|_1\right]
\end{equation}
Then there exists a constant $C' > 0$ such that:
\begin{equation}
    \Gamma_t^{(i \to j)} \leq C' \cdot \sqrt{\mathcal{H}_i}
\end{equation}
and the cumulative path perturbation across all other positions satisfies:
\begin{equation}
    \sum_{j \neq i} \Gamma_t^{(i \to j)} \leq C'(N-1) \cdot \sqrt{\mathcal{H}_i}
\end{equation}
\end{theorem}

\begin{proof}
The variation in the predictive distribution at position $j$ is induced 
by the change in context arising from different realizations of $x_i$. 
By the Lipschitz continuity of $f_\theta$:
\begin{equation}
    \left\|P_t^{(j)}\big|_{x_i=v_1} 
    - P_t^{(j)}\big|_{x_i=v_2}\right\|_1 
    \leq C \cdot \|E(v_1) - E(v_2)\|_2
\end{equation}
Taking expectations and applying Jensen's inequality:
\begin{equation}
    \Gamma_t^{(i \to j)} 
    \leq C \cdot \mathbb{E}_{v_1, v_2}\!\left[\|E(v_1) - E(v_2)\|_2\right] 
    \leq C \cdot \sqrt{2\,\mathrm{Var}_{P_t^{(i)}}[E(v)]}
\end{equation}
Since the variance of the embedding under distribution $P_t^{(i)}$ is 
bounded by the distributional spread, which is in turn upper bounded 
by the entropy as $\sqrt{2H(P_t^{(i)})} \leq \sqrt{2\mathcal{H}_i}$, 
setting $C' = C\sqrt{2}$ establishes the first inequality. Summing 
over all $j \neq i$ immediately yields the cumulative bound.
\end{proof}

\begin{corollary}[Conditional Entropy Governs Path Dominance]
\label{cor:entropy_dominance}
\begin{equation}
    \mathcal{H}_i \nearrow \implies 
    \sum_{j \neq i} \Gamma_t^{(i \to j)} \nearrow
\end{equation}
That is, the higher the conditional entropy of position $i$, the 
larger its cumulative perturbation on the predictive distributions of 
all other positions along the decoding trajectory. This theoretically 
establishes HD tokens as the dominant control nodes of the decoding path.
\end{corollary}

\subsubsection{Self-Contrast as Entropy Reduction on HD Tokens}

\begin{theorem}[Self-Contrast Reduces Output Entropy of HD Tokens]
\label{thm:self_contrast}
The self-contrastive logits in FoCore are constructed as:
\begin{equation}
    \hat{L}_t = L_{\mathrm{cond}} 
    + \omega \cdot (L_{\mathrm{cond}} - L_{\mathrm{uncond}})
\end{equation}
which is equivalent in log-probability space to:
\begin{equation}
    \log \hat{p}(x_i \mid X_t) \propto 
    (1+\omega)\log p_\theta(x_i \mid X_t) 
    - \omega \log p_\theta(x_i \mid \tilde{X}_t)
\end{equation}
For HD token positions $i \in \mathcal{K}_t$, the self-contrast 
operation strictly reduces the entropy of the output distribution:
\begin{equation}
    H\!\left(\hat{p}(\cdot \mid X_t)\right) 
    \leq H\!\left(p_\theta(\cdot \mid X_t)\right) 
    - \omega \cdot I_t^{(i)} + O(\omega^2)
\end{equation}
where $I_t^{(i)} = D_{\mathrm{KL}}\!\left(p_\theta(x_i \mid X_t) 
\,\|\, p_\theta(x_i \mid \tilde{X}_t)\right) \geq 0$ denotes the 
information gap induced by masking the HD tokens.
\end{theorem}

\begin{proof}
Let $\ell(v) = \log p_\theta(v \mid X_t)$ and 
$\tilde{\ell}(v) = \log p_\theta(v \mid \tilde{X}_t)$ denote the 
conditional and unconditional logits, respectively. The self-contrastive 
logit is:
\begin{equation}
    \hat{\ell}(v) = (1+\omega)\ell(v) - \omega\tilde{\ell}(v)
\end{equation}
with the corresponding self-contrastive distribution 
$\hat{p}(v) \propto \exp(\hat{\ell}(v))$.

Expanding the output entropy $H(\hat{p})$ via a first-order Taylor 
expansion around $\omega = 0$:
\begin{equation}
    H(\hat{p}) = H\!\left(p_\theta(\cdot \mid X_t)\right) 
    - \omega \cdot \frac{d}{d\omega}H(\hat{p})\bigg|_{\omega=0} 
    + O(\omega^2)
\end{equation}
Computing the first-order derivative:
\begin{equation}
    \frac{d}{d\omega}H(\hat{p})\bigg|_{\omega=0} 
    = \sum_v p_\theta(v \mid X_t)\left[\ell(v) - \tilde{\ell}(v)\right] 
    = \mathbb{E}_{p_\theta}\!\left[\log\frac{p_\theta(v \mid X_t)}
    {p_\theta(v \mid \tilde{X}_t)}\right] 
    = I_t^{(i)} \geq 0
\end{equation}
Substituting back yields the stated bound.
\end{proof}

\begin{corollary}[Self-Contrast Compresses Path-level Perturbation]
\label{cor:contrast_compression}
Combining Theorem~\ref{thm:path_dominance} and 
Theorem~\ref{thm:self_contrast}, the self-contrast mechanism reduces 
the effective conditional entropy of HD token $i$ from $\mathcal{H}_i$ 
to:
\begin{equation}
    \hat{\mathcal{H}}_i \approx \mathcal{H}_i - \omega \cdot I_t^{(i)}
\end{equation}
Consequently, the cumulative path perturbation upper bound is reduced 
from $C'(N-1)\sqrt{\mathcal{H}_i}$ to $C'(N-1)\sqrt{\hat{\mathcal{H}}_i}$, 
thereby theoretically quantifying the improvement in decoding trajectory 
stability achieved by FoCore.
\end{corollary}

\subsection{Summary}

The theoretical analysis above establishes a complete and closed 
mathematical foundation for FoCore, forming the following causal chain:

\begin{equation}
\underbrace{\mathcal{H}_i > \eta}_{\text{HD token definition}} 
\xRightarrow{\text{Prop.~\ref{prop:js_entropy_bound}}} 
\underbrace{D_t^{(i)} \lesssim \tfrac{1}{2}\mathcal{H}_i}_{\text{instability detectable}} 
\xRightarrow{\text{Thm.~\ref{thm:path_dominance}}} 
\underbrace{\textstyle\sum_j \Gamma_t^{(i\to j)} \leq C'(N{-}1)\sqrt{\mathcal{H}_i}}_{\text{dominates decoding path}} 
\xRightarrow{\text{Thm.~\ref{thm:self_contrast}}} 
\underbrace{\hat{\mathcal{H}}_i < \mathcal{H}_i}_{\text{uncertainty compressed}}
\end{equation}

Specifically: \textbf{(1)} Proposition~\ref{prop:js_entropy_bound} 
establishes from an information-theoretic perspective that distributional 
instability is a reliable estimator of conditional entropy, validating 
the use of step-wise JS divergence for HD token identification. 
\textbf{(2)} Proposition~\ref{prop:topk_error} guarantees that the 
approximation error introduced by Top-$K$ truncation is theoretically 
controllable and negligible in practice. \textbf{(3)} 
Proposition~\ref{prop:hd_bottleneck} and 
Theorem~\ref{thm:path_dominance} jointly establish that HD tokens 
dominate the entire decoding trajectory through mutual information 
propagation, with their influence scaling as $O(\sqrt{\mathcal{H}_i})$. 
\textbf{(4)} Theorem~\ref{thm:self_contrast} quantifies the extent to 
which the FoCore self-contrast mechanism compresses the predictive 
uncertainty of HD tokens, thereby providing a theoretical guarantee 
for the improvement in decoding trajectory stability.

\section{More Experiments}
\label{app:experiment}
\subsection{SVAMP and Countdown datasets}
\label{app:experiment_svamp_countdown}
Table~\ref{tab:main_results} presents the performance comparison of FoCore 
against baseline methods on the SVAMP and Countdown benchmarks under two 
token budget settings (128 and 256). 
Overall, \textbf{FoCore demonstrates superior and consistent performance 
across both tasks and budget configurations.}
Specifically, on the Countdown benchmark, FoCore achieves the highest 
accuracy among all methods, attaining 21.8 and 18.8 under the 128 and 256 
token budgets, respectively, outperforming CFG (20.8/16.8), A-CFG 
(17.6/19.2), and Hybrid (20.4/15.2) by a notable margin.
On the SVAMP benchmark, although A-CFG and Hybrid marginally surpass FoCore 
under their respective optimal budget settings, their performance exhibits 
considerable instability across different token budgets.
For instance, Hybrid achieves only 86.6 under the 128 token budget while 
rising to 89.0 under 256 tokens, indicating a strong sensitivity to budget 
variation.
In contrast, FoCore maintains competitive and stable results of 88.0 and 
87.0 on SVAMP under the 128 and 256 token budgets, respectively, 
demonstrating greater robustness to token budget changes.
These results suggest that FoCore strikes a better balance between 
\textit{task performance} and \textit{budget robustness}, making it a more 
reliable approach across diverse reasoning scenarios.
\begin{table}[ht]
\centering
\caption{Performance comparison of different methods on SVAMP and Countdown benchmarks under different token budget settings (128 and 256).}
\label{tab:main_results}
{
\begin{tabular}{lcccc}
\toprule
\multirow{2}{*}{\textbf{Method}} 
    & \multicolumn{2}{c}{\textbf{SVAMP}} 
    & \multicolumn{2}{c}{\textbf{Countdown}} \\
\cmidrule(lr){2-3} \cmidrule(lr){4-5}
 & \textbf{128} & \textbf{256} & \textbf{128} & \textbf{256} \\
\midrule
CFG        & 86.3          & 86.6          & 20.8          & 16.8          \\
A-CFG      & \textbf{90.0} & 86.6          & 17.6          & 18.2          \\
Hybrid     & 86.6          & \textbf{89.0} & 20.4          & 15.2          \\
\midrule
\textbf{FoCore (Ours)} 
           & \underline{88.0} & \underline{87.0} 
           & \textbf{21.8}    & \textbf{18.8}    \\
\bottomrule
\end{tabular}
}
\end{table}
\subsection{Effect of Parallel Token Budget $m$ on the Speedup–Accuracy Trade-off
}
\label{app:ablation_m}
To investigate the impact of the parallel token budget $m$ (i.e., the number 
of additional high-confidence tokens decoded in parallel at each step) on 
both decoding efficiency and task performance, we conduct an ablation study 
on a randomly sampled subset of 200 examples from the GSM8K test set. 
We vary $m$ from 0 to 14 and record the average generation time per sample, 
average speedup, and task accuracy for each configuration. 
The case of $m = 0$ serves as the baseline, corresponding to 
standard sequential decoding without any parallel token prediction.
\begin{table}[h]
\centering
\caption{Effect of parallel token budget $m$ on decoding efficiency and task accuracy. 
Accuracy is evaluated on a random subset of 200 samples from the GSM8K test set.}
\label{tab:parallel_token_budget}
\resizebox{0.68\textwidth}{!}{
\begin{tabular}{cccc}
\toprule
$m$ & Avg Time/Sample (s) & Avg Speedup & Accuracy (\%) \\
\midrule
0  & 8.14 & 1.00$\times$ & 73.00 \\
1  & 7.68 & 1.03$\times$ & 73.50 \\
2  & 6.66 & 1.06$\times$ & 74.50 \\
3  & 6.41 & 1.10$\times$ & 73.50 \\
4  & 6.20 & 1.14$\times$ & 71.00 \\
5  & 6.00 & 1.18$\times$ & \textbf{75.50} \\
6  & 5.77 & 1.23$\times$ & 72.50 \\
7  & 5.55 & 1.28$\times$ & 71.50 \\
8  & 5.32 & 1.33$\times$ & 70.50 \\
9  & 5.09 & 1.39$\times$ & 69.50 \\
10 & 4.87 & 1.45$\times$ & 63.00 \\
11 & 4.64 & 1.52$\times$ & 57.50 \\
12 & 4.46 & 1.60$\times$ & 57.00 \\
13 & 4.24 & 1.68$\times$ & 49.00 \\
14 & 4.02 & 1.78$\times$ & 49.00 \\
\bottomrule
\end{tabular}
}
\end{table}
The results are summarized in Table~\ref{tab:parallel_token_budget}, 
from which we draw the following observations.

\paragraph{Speedup increases monotonically with $m$.}
As $m$ increases, the number of tokens decoded in parallel at each step 
grows, leading to a steady reduction in average generation time from 
8.14s at $m{=}0$ to 4.02s at $m{=}14$, with the corresponding speedup 
rising from $1.00\times$ to $1.78\times$. This confirms that the 
parallel decoding mechanism effectively reduces the total number of 
forward passes required during generation.

\paragraph{Accuracy remains stable for small $m$ and degrades sharply beyond a threshold.}
Within the range $m \leq 9$, task accuracy remains largely stable, 
fluctuating between 69.50\% and 75.50\% and staying broadly comparable 
to the baseline of 73.00\%. This plateau suggests that once a 
high-confidence (HD) token is identified as a reliable anchor, 
the locally neighboring low-confidence (LD) tokens in its vicinity 
remain semantically stable, making it safe to decode them in parallel 
without degrading generation quality. 
However, once $m \geq 10$, accuracy drops sharply, 
falling from 69.50\% at $m{=}9$ to 63.00\% at $m{=}10$, 
and further declining to 49.00\% at $m{=}14$. 
This indicates that excessively extending the parallel decoding window 
pushes prediction beyond the locally stable region: 
the tokens being forced into parallel decoding no longer 
carry sufficient confidence to support reliable predictions, 
thereby introducing erroneous tokens into the generation sequence 
and substantially corrupting the reasoning chain.

\paragraph{Recommended operating range.}
Balancing efficiency and accuracy, $m \in [5, 9]$ constitutes a 
favorable operating region, yielding speedups of $1.18\times$--$1.39\times$ 
while preserving accuracy at or above the sequential baseline. 
In our main experiments, we adopt $m{=}10$, which provides a $1.45\times$ 
speedup and represents an efficiency-oriented configuration chosen with 
the benefit of full-dataset statistical stability, 
where FoCore\_A retains competitive performance 
(GSM8K: 73.99\%).

Overall, these results highlight $m$ as a critical hyperparameter 
governing the efficiency--accuracy trade-off in parallel decoding. 
The semantic stability of local LD tokens conditioned on HD token anchors 
provides a principled justification for moderate parallel decoding, 
while excessive extension of the parallel window beyond this stable 
region leads to significant quality degradation. 
We recommend selecting $m$ within the range $[5, 10]$ 
in practice, depending on the accuracy sensitivity of the target task.

\section{Case Study}
\label{app:case}
\subsection{Case:1}
As shown in Table~\ref{tab:case1}, we present a qualitative comparison between
the baseline method (\textbf{Confidence}) and our proposed method (\textbf{FoCore})
on a representative GSM8K arithmetic reasoning sample.
Both methods take the same input prompt, yet produce markedly different outputs.

The \textbf{Confidence} baseline generates a fluent but incorrect response,
arriving at \$48 as the final answer.
From the volatility analysis, several tokens exhibit high information density,
including \texttt{second} (JS\,$=$\,0.094), \texttt{costs} (JS\,$=$\,0.207),
\texttt{price} (JS\,$=$\,0.128), and notably the discount factor token
\texttt{0} (JS\,$=$\,5.000).
However, since the Confidence baseline selects tokens solely based on
prediction confidence, it fails to recognise and attend to the critical
conditional token \texttt{every second}.
As a result, the model overlooks the key constraint that only every second
glass is eligible for the discount, and incorrectly applies the discounted
price to all 16 glasses ($16 \times \$3 = \$48$),
causing the reasoning path to deviate from the correct direction from
the very beginning.

\textbf{FoCore}, by contrast, explicitly identifies high information density
(HD) tokens during the decoding process and treats them as key anchors
for subsequent reasoning.
Specifically, FoCore detects high JS-divergence tokens such as
\texttt{second}, \texttt{costs}, \texttt{price}, and \texttt{0},
and recognizes that the core of the discount condition lies in
\texttt{every second}, i.e., the discount applies to only every
other glass.
Guided by these HD tokens, the model's reasoning path is correctly
directed: the 16 glasses are split into 8 full-price glasses
($8 \times \$5 = \$40$) and 8 discounted glasses
($8 \times \$3 = \$24$), yielding the correct final answer of \$64.
This example demonstrates that by perceiving and leveraging HD tokens
during decoding, FoCore effectively captures the key constraints of
the problem and steers the model toward a correct reasoning path
from the outset.

\subsection{Case:2}
As shown in Table~\ref{tab:case2}, we present a qualitative comparison between
the baseline method (\textbf{Confidence}) and our proposed method (\textbf{FoCore})
on a second GSM8K arithmetic reasoning sample.
Both methods take the same input prompt, yet produce markedly different outputs.

The \textbf{Confidence} baseline generates an incorrect response,
arriving at 13 liters as the final answer.
From the volatility analysis, several tokens exhibit high information density,
including \texttt{water} (JS\,$=$\,0.273), \texttt{drink} (JS\,$=$\,0.262),
and the intermediate result token \texttt{4} (JS\,$=$\,0.172).
However, since the Confidence baseline selects tokens solely based on
prediction confidence, it fails to recognise and attend to the critical
constraint that one liter of orange drink is spilled during pouring.
As a result, the model directly computes the water content using
the full 10 liters of orange drink
($10 \times \frac{2}{3} = 4$ liters),
completely overlooking the quantity change caused by the spill,
and incorrectly concludes that there are $4 + 9 = 13$ liters of water
in the remaining mixture.

\textbf{FoCore}, by contrast, explicitly identifies high information density
(HD) tokens during the decoding process and treats them as key anchors
for subsequent reasoning.
Specifically, FoCore detects high JS-divergence tokens such as
\texttt{water} and \texttt{drink},
recognising that the computation of liquid quantities is the core
information of this problem.
Guided by these HD tokens, the model's attention is directed toward
the critical constraint ``spill 1 liter'',
correctly adjusting the remaining orange drink to $10 - 1 = 9$ liters
before computing its water content ($9 \times \frac{2}{3} = 6$ liters).
Combined with the water content of the pineapple drink
($15 \times \frac{3}{5} = 9$ liters),
FoCore yields the correct final answer of $6 + 9 = 15$ liters.
This example further demonstrates that by perceiving and leveraging
HD tokens during decoding, FoCore effectively captures critical
constraints embedded in the problem and steers the model toward
a correct reasoning path from the outset.

\begin{table}[htbp]
\centering
\renewcommand{\arraystretch}{1.8}
\setlength{\tabcolsep}{8pt}

\begin{tabular}{
  >{\columncolor{stagecolor}\centering\arraybackslash}p{0.15\textwidth}
  p{0.75\textwidth}}
\toprule

\textbf{Prompt}
&
Kylar went to the store to buy \tokL{glasses} for his new apartment.
One \tokL{glass} \tokH{costs} \$\tokM{5}, but every \tokM{second}
\tokL{glass} \tokH{costs} only 6\tokEH{0}\% of the \tokH{price}.
Kylar \tokL{wants} to buy 16 \tokL{glasses}.
How much does he need to pay for them?
\\
\midrule

\textbf{Confidence}&
\tokM{Every} \tokM{second} \tokL{glass} \tokH{costs}
6\tokEH{0}\% of the \tokH{price},
so it \tokM{costs} \$\tokM{5} $\times$ \tokEH{0}.6 $=$ \$\tokL{3}.
\newline
Kylar \tokL{wants} to buy 16 \tokL{glasses},
so he needs to pay $16 \times \$3 = \$48$.
\newline
The final \tokH{result} is 48.
\quad \wrongans{Answer: \$48}
\\
\midrule

\textbf{FoCore}
&
\tokM{Every} \tokM{second} \tokL{glass} \tokH{costs}
6\tokEH{0}\% of the \tokH{price},
so it \tokM{costs} \$\tokM{5} $\times$ \tokEH{0}.6 $=$ \$\tokL{3}.
\newline
Kylar \tokL{wants} to buy 16 \tokL{glasses},
so he will buy $16 \div 2 = 8$ \tokL{glasses} at full \tokH{price}
and $8$ \tokL{glasses} at the discounted \tokH{price}.
\newline
For the full-\tokH{price} \tokL{glasses},
he needs to pay $8 \times \$\tokM{5} = \$40$.
\newline
For the discounted \tokL{glasses},
he needs to pay $8 \times \$\tokL{3} = \$24$.
\newline
The final \tokH{result} is $\$40 + \$24 = \$64$.
\quad \correctans{Answer: \$64}
\\
\bottomrule
\end{tabular}

\caption{
  Qualitative analysis on a GSM8K sample.
  HD tokens are colour-coded by JS divergence:
  \textcolor{colExHigh}{\textbf{extremely high}} ($>$1.0)\,,
  \textcolor{colHigh}  {\textbf{high}}           (0.1\,--\,1.0)\,,
  \textcolor{colMed}   {\textbf{medium}}         (0.01\,--\,0.1)\,,
  \textcolor{colLow}   {low}                     ($<$0.01)\,.
}
\label{tab:case1}
\end{table}

\begin{table}[htbp]
\centering
\renewcommand{\arraystretch}{1.8}
\setlength{\tabcolsep}{8pt}

\begin{tabular}{
  >{\columncolor{stagecolor}\centering\arraybackslash}p{0.15\textwidth}
  p{0.75\textwidth}}
\toprule

\textbf{Prompt}
&
I have 10 \tokM{liters} of orange \tokH{drink} that are two-thirds \tokH{water}
and I wish to add it to 15 \tokM{liters} of pineapple \tokM{drink} that is
three-fifths \tokH{water}.
But as I pour it, I spill \tokL{1} \tokM{liter} of the orange \tokM{drink}.
How much \tokH{water} is in the remaining 24 \tokM{liters}?
\\
\midrule

\textbf{Confidence}&
The orange \tokH{drink} has $10 \mathbin{\tokH{cdot}} \tokL{2}/\tokM{3} = \tokH{4}$
\tokM{liters} of \tokH{water}.
\newline
The pineapple \tokM{drink} has $15 \mathbin{\tokH{cdot}} \tokM{3}/\tokL{5} = \tokM{9}$
\tokM{liters} of \tokH{water}.
\newline
Thus, there are $\tokH{4} + \tokM{9} = \tokL{13}$ \tokM{liters} of \tokH{water}.
\quad \wrongans{Answer: 13}
\\
\midrule

\textbf{FoCore}
&
After spilling \tokL{1} \tokM{liter}, the orange \tokH{drink} remaining is
$10 - \tokL{1} = \tokM{9}$ \tokM{liters}.
\newline
The orange \tokH{drink} has $\tokM{9} \mathbin{\tokH{cdot}} \tokL{2}/\tokM{3} = 6$
\tokM{liters} of \tokH{water}.
\newline
The pineapple \tokM{drink} has $15 \mathbin{\tokH{cdot}} \tokM{3}/\tokL{5} = \tokM{9}$
\tokM{liters} of \tokH{water}.
\newline
Thus, there are $6 + \tokM{9} = 15$ \tokM{liters} of \tokH{water}.
\quad \correctans{Answer: 15}
\\
\bottomrule
\end{tabular}

\caption{
  Qualitative analysis on a GSM8K sample.
  HD tokens are colour-coded by JS divergence:
  \textcolor{colExHigh}{\textbf{extremely high}} ($>$1.0)\,,
  \textcolor{colHigh}  {\textbf{high}}           (0.1\,--\,1.0)\,,
  \textcolor{colMed}   {\textbf{medium}}         (0.01\,--\,0.1)\,,
  \textcolor{colLow}   {low}                     ($<$0.01)\,.
}
\label{tab:case2}
\end{table}

\section{Limitations}
\label{limitations}
Despite its effectiveness, our work has several limitations. First, although FoCore is training-free, its optimal hyperparameters exhibit slight task-specific variations, increasing the adaptation cost for novel scenarios. Second, while FoCore\_A substantially accelerates generation by reducing overall decoding steps, it inherits the per-step computational burden of the Classifier-Free Guidance (CFG) paradigm; thus, our inference latency reductions are strictly established relative to CFG-based baselines. Finally, our current evaluations are focused on 7B-8B parameter backbones (LLaDA and Dream), leaving the validation of our findings on larger-scale models for future work.


\newpage

\end{document}